\documentclass{article}

\usepackage[main, final]{neurips_2026}

\usepackage[utf8]{inputenc} 
\usepackage[T1]{fontenc}    
\usepackage{hyperref}       
\usepackage{url}            
\usepackage{booktabs}       
\usepackage{amsfonts}       
\usepackage{nicefrac}       
\usepackage{microtype}      
\usepackage{xcolor}         
\usepackage{subcaption}
\usepackage{mathtools,latexsym}
\usepackage{amsthm,thmtools,thm-restate}
\usepackage{wrapfig}

\usepackage[utf8]{inputenc} 
\usepackage[T1]{fontenc}    
\PassOptionsToPackage{hyphens}{url}
\usepackage{hyperref} 
\hypersetup{colorlinks,
  linkcolor={red!50!black},
  citecolor={blue!50!black},
  urlcolor={blue!5!black}}

\usepackage{url}            
\usepackage{booktabs}       
\usepackage{amsfonts}       
\usepackage{bbm}
\usepackage{amssymb}
\usepackage{nicefrac}       
\usepackage{microtype}      

\usepackage[capitalize]{cleveref}
\usepackage{csquotes}

\newcommand{\loss}{\ell}

\newcommand{\predictionset}{\mathcal{C}}
\newcommand{\calibration}[1][]{\mathcal{D}_{#1}}
\newcommand{\threshold}{\lambda}
\newcommand{\thresholdmax}{\lambda_{\max}}

\newcommand{\Threshold}{\Lambda}
\newcommand{\Risk}{R}

\newcommand{\correction}{\gamma}
\newcommand{\R}{\mathbb{R}}
\newcommand{\E}{\mathbb{E}}
\newcommand{\sample}[1][]{(X_{#1},Y_{#1})}
\newcommand{\samples}[1][]{(X_{i},Y_i)_{i=1}^{#1}}

\definecolor{mygreen}{rgb}{0.21, 0.60, 0.21}

\theoremstyle{plain}
\newtheorem{theorem}{Theorem}[section]
\newtheorem{proposition}[theorem]{Proposition}
\newtheorem{lemma}[theorem]{Lemma}
\newtheorem{corollary}[theorem]{Corollary}
\theoremstyle{definition}
\newtheorem{definition}[theorem]{Definition}

\theoremstyle{remark}
\newtheorem{remark}[theorem]{Remark}

\theoremstyle{remark}
\newtheorem{example}[theorem]{Example}

\usepackage[textsize=tiny]{todonotes}

\title{Anytime-Valid Conformal Risk Control}

\author{%
  Bror Hultberg\\
  Uppsala University\\
  \texttt{bror.hultberg@it.uu.se} \\
  \And
  Dave Zachariah \\
   Uppsala University \\
  \texttt{dave.zachariah@it.uu.se} \\
  \And
  Ant\^onio H. Ribeiro \\
  Uppsala University \& SciLifeLab\\
  \texttt{antonio.horta.ribeiro@it.uu.se} \\
}

\begin{document}

\maketitle

\begin{abstract}
Prediction sets provide a means of quantifying the uncertainty in predictive tasks. Using held out calibration data, conformal prediction and risk control can produce prediction sets that exhibit statistically valid error control in a computationally efficient manner. However, in the standard formulations, the error is only controlled on average over many possible calibration datasets of fixed size. In this paper, we extend the control to remain valid with high probability over a cumulatively growing calibration dataset at any time point. We derive such guarantees using quantile-based arguments and illustrate the applicability of the proposed framework to settings involving distribution shift. We further show that our guarantees are asymptotically tight, in the sense that the risks converge to the target level from below, and that the rate of our correction cannot be improved. 
Finally, we demonstrate the practical performance of our methods through both simulations and real-world numerical examples.
\end{abstract}

\section{Introduction}

Standard prediction problems focus on a point prediction $f(X)$ of an outcome variable $Y$. A prediction \emph{set} $\predictionset(X)$, which aims to cover $Y$, is often more informative as the set size reflects uncertainty of the outcome. Thus $\predictionset(X)$ could be a variable interval or a collection of labels depending on the domain of the outcome. Split conformal prediction is a widely used framework for constructing prediction sets. Using a held out calibration set  $\calibration[n]=\samples[n]$~\citep{vovk2005algorithmic,angelopoulos2023conformal}, the framework generates a prediction set $\predictionset_n(X)$ that ensures a miscoverage rate no greater than $\alpha$ for the next observation $(X_{n+1}, Y_{n+1})$, i.e.,
\begin{equation}
    \mathbb{P}\big(Y_{n+1} \not \in \predictionset_n(X_{n+1}) \big) \leq \alpha.
\label{eq:miscoveragecontrol_standard}
\end{equation}
While elegant and computationally simple, the achievable miscoverage control \eqref{eq:miscoveragecontrol_standard} has certain limitations. 

One main limitation about the procedure is that it is a data-marginalized control, meaning that the probability is evaluated \emph{jointly} over the calibration set $\calibration[n]$ and the out-of-sample test point $\sample[n+1]$. Consequently, for any realized  dataset $\calibration[n]$, the miscoverage rate can vary widely from $\alpha$ \citep{pmlr-v25-vovk12,MARQUESF2025110350,zwart2025probabilisticconformalcoverageguarantees}. For instance, when $n=500$ and $\alpha=0.1$ there is a 48\% chance of drawing $\calibration[n]$ such that the rate exceeds the specified $\alpha$, see Figure \ref{fig:1a}. This limitation is well known, and \emph{sample-conditional} (or PAC-style) guarantees address it by ensuring that the miscoverage rate of the realized prediction set is at most $\alpha$ with a given probability over the draw of $\calibration[n]$, say, $1-\delta = 90\%$ \citep{pmlr-v25-vovk12,RCPS10.1145/3478535,duchi2025sampleconditional}, see Figure~\ref{fig:fixed}. Such guarantees, however, are stated for a calibration set of a single, fixed size $n$.

In many circumstances data collection is a cumulative process such that $n$ in $\calibration[n]$ increases and yields a sequence of prediction sets $\{ \predictionset_n(X) \}$. However, \eqref{eq:miscoveragecontrol_standard} and its fixed-time sample-conditional counterparts only control the miscoverage rate for an individual $\predictionset_n(X)$ rather than all the rates for the \emph{entire} sequence of prediction sets with high probability. In fact, by the law of the iterated logarithm and under mild conditions, these corrections lead to miscoverage rates that exceed $\alpha$ for infinitely many $n$ almost surely (Proposition~\ref{prop:rateoptimal}), as reflected by the many red trajectories in Figure~\ref{fig:fixed}. Overall, we want to ensure that the miscoverage rates approach $\alpha$ as $n$ increases, while remaining below it for all $n$, with high probability (Figure~\ref{fig:1b}).

While miscoverage control is important in certain applications, there are other relevant types of error control for $\{ \predictionset_n(X) \}$ that we want to consider, such as false negative rate control in multilabel classification. The chosen error metric is the \emph{risk} of a prediction set \citep{angelopoulos2024conformal}.

In this paper, we propose an approach that jointly  addresses all these limitations. Our contributions are as follows:
\begin{enumerate}
    \item We show that one can construct a sequence of prediction sets $\{ \predictionset_n(X) \}$ for which the risks will be no greater than $\alpha$ and approach it from below with high probability.
    \item We show that, under mild conditions, the risks converge to $\alpha$ from below at rate $O(\sqrt{n^{-1}\log\log n})$, and we prove that this rate cannot be improved. Any correction term that decays faster fails to achieve anytime-valid risk control almost surely.
    \item 
    Finally, we specialize our results to the case of test-time distribution shifts and illustrate them with numerical experiments.
\end{enumerate}

\begin{figure*}
    \centering
    \begin{subfigure}[b]{0.32\textwidth}
        \centering
        \includegraphics[width=\textwidth]{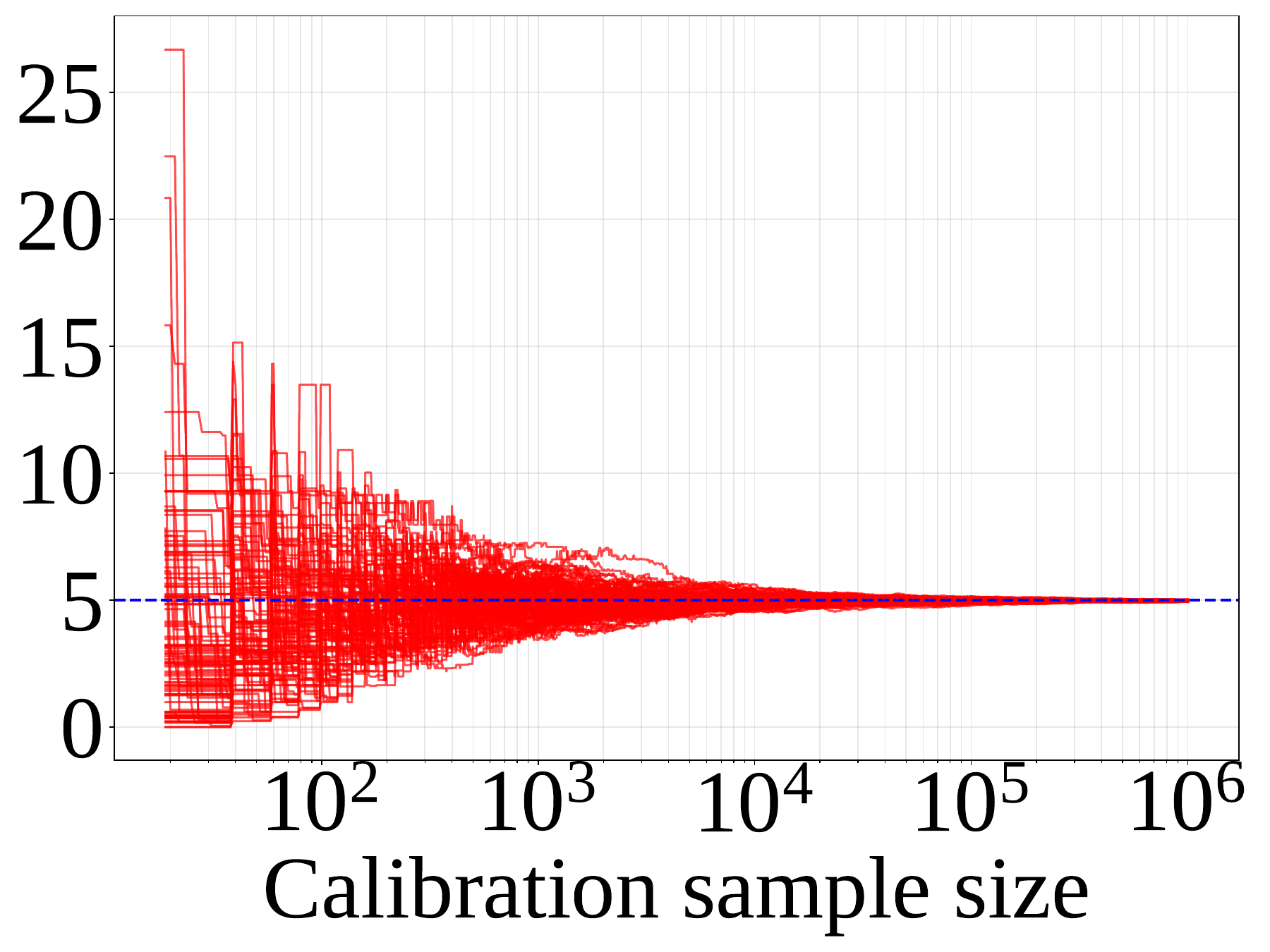} 
        \caption{Standard method}
        \label{fig:1a}
    \end{subfigure}
    \hfill
    \begin{subfigure}[b]{0.32\textwidth}
        \centering
        \includegraphics[width=\textwidth]{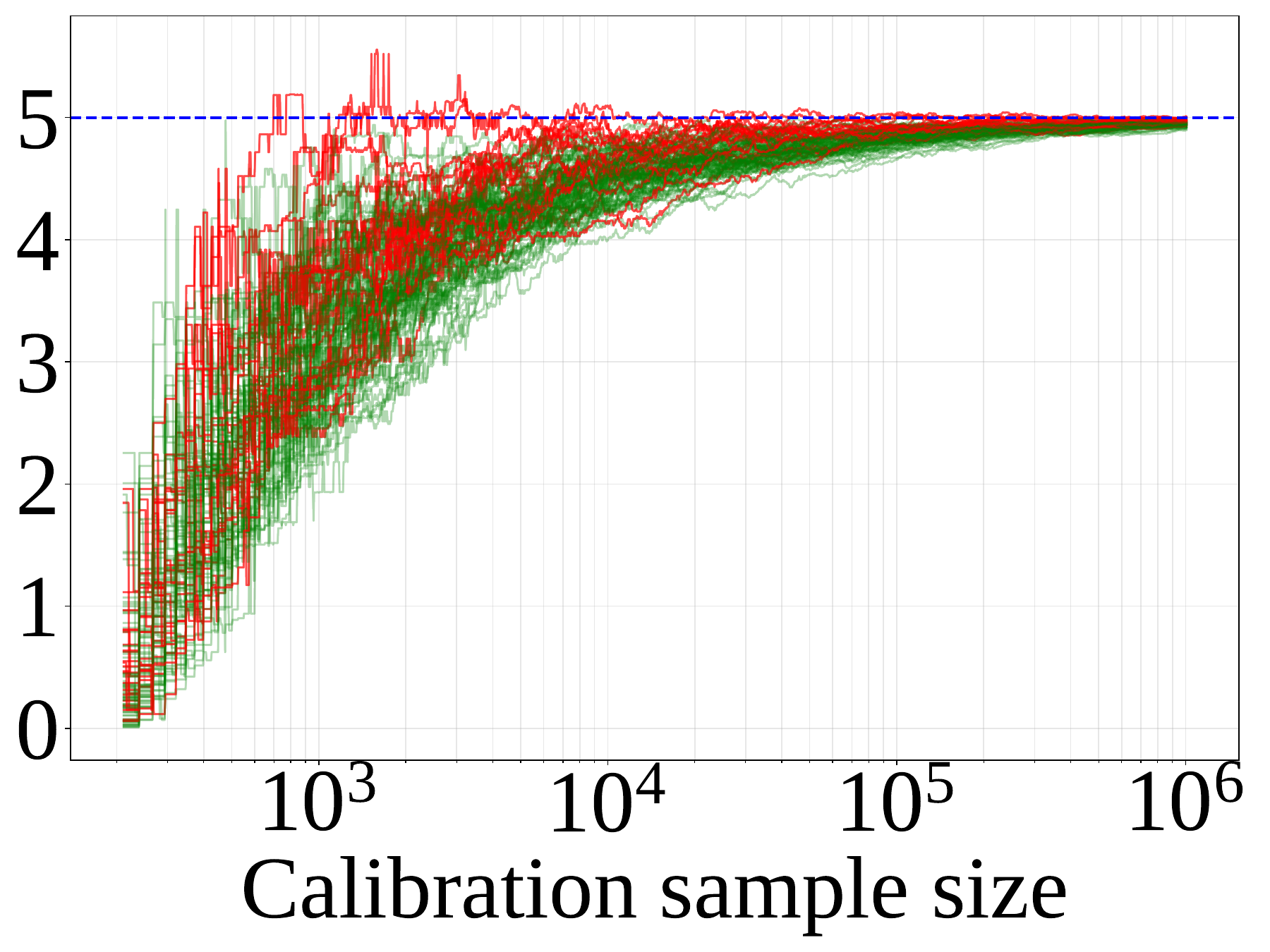} 
        \caption{Fixed-time valid, $\delta=10\%$.}
        \label{fig:fixed}
    \end{subfigure}
    \hfill
    \begin{subfigure}[b]{0.32\textwidth}
        \centering
        \includegraphics[width=\textwidth]{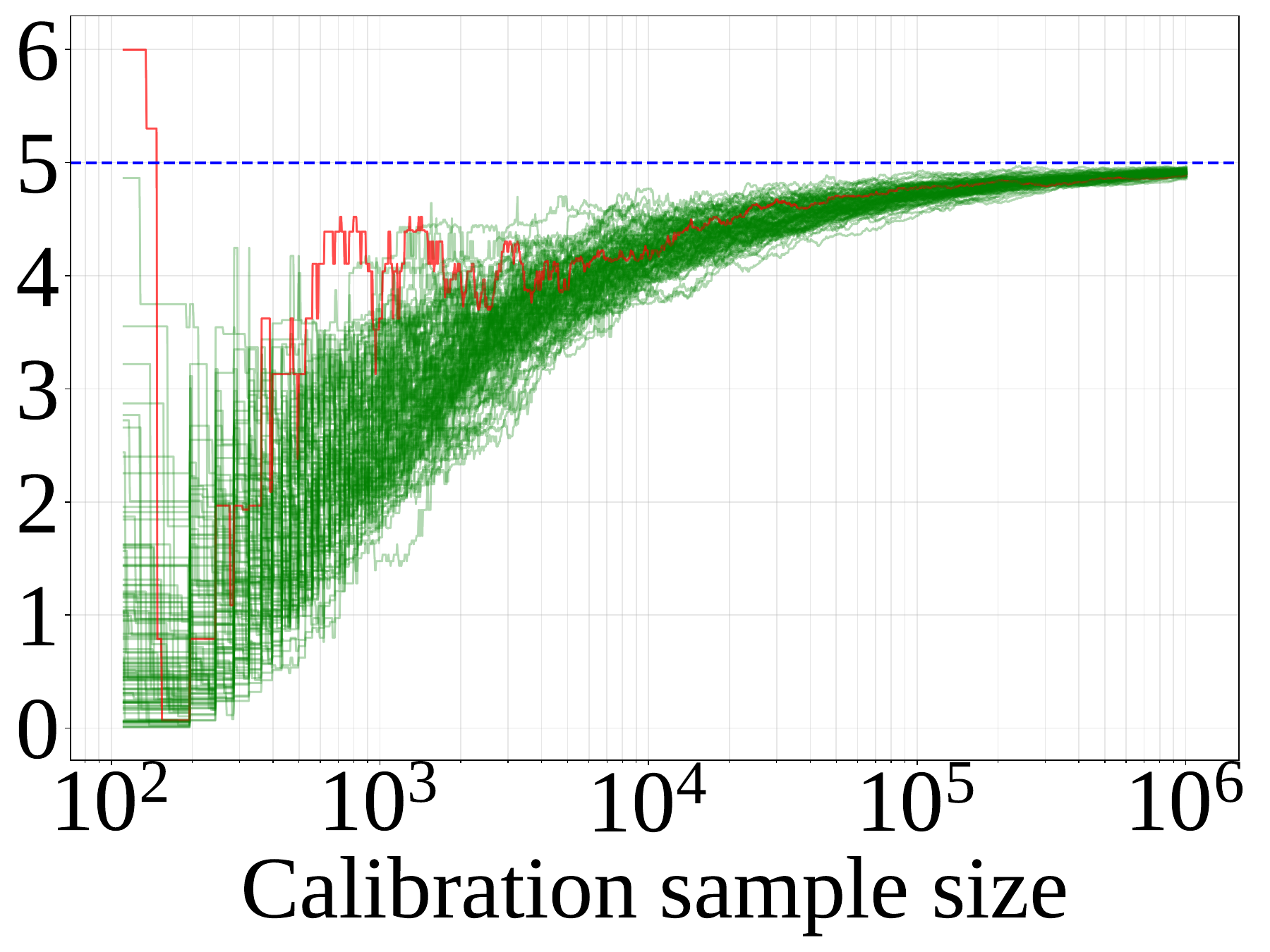} 
        \caption{Anytime-valid, $\delta=10\%$.}
        \label{fig:1b}
    \end{subfigure}
    \caption{Miscoverage rates of prediction sets $\{ \predictionset_n(X) \}$ 
    versus calibration sample size $n$. Each line corresponds to a particular 
    draw of calibration data and the miscoverage rates are evaluated conditional 
    on this data at each $n$. A \textcolor{mygreen}{\bf green} line indicates 
    that \emph{all} the rates fall below $\alpha = 5\%$, while 
    \textcolor{red}{\bf red} indicates a failure. Details are given in 
    Section~\ref{sec:synthetic}.}
    \label{fig:1}
\end{figure*}

\section{Problem Formulation}

A prediction set is a subset of the domain of the outcome variable, i.e., $\predictionset(X) \subseteq \mathcal{Y}$. Throughout the paper, the index $n \in \mathbb{N}$ denotes the number of calibration points collected so far, and $\calibration[n] = \samples[n]$ denotes the calibration set of size $n$. We assume that the calibration points $(X_1, Y_1), (X_2, Y_2), \dots$ and test point $(X,Y)$ are drawn independently and identically. We let $\predictionset_n(X)$ be a function of the calibration data $\calibration[n]$. 

\subsection{Risk of a Prediction Set}

The prediction error of $\predictionset(X)$ is quantified by a loss function $\loss(\predictionset(X),Y)$. The \emph{risk} of the \emph{given} prediction set $\predictionset_n(X)$ is defined as:
\begin{equation}
\E[\loss\left(\predictionset_{n}(X),Y\right) \: | \: \calibration[n] ],
\label{eq:risk_dataconditional}
\end{equation}
which is a sample-conditional risk.
\begin{example}
In regression and single-label classification problems, the  miscoverage event is a common loss function $\loss(\predictionset(X),Y) = \mathbf{1}\{Y \notin \predictionset(X)\}$. In this case, \eqref{eq:risk_dataconditional} is the miscoverage rate of the given prediction set, i.e.,
\begin{equation}
\mathbb{P}(Y \not \in \predictionset_{n}(X) \: | \: \calibration[n]),
\label{eq:miscoveragerate_dataconditional} 
\end{equation}
\end{example}

\begin{example}
\label{ex:FNR}
In multilabel classification problems, where $\mathcal{Y}=\{1,\dots,K\}$ denotes the labels, the fraction of true labels missed by the prediction set $\loss(\predictionset(X),Y)= 1-|\predictionset(X) \cap Y|/|Y|$
is a suitable loss function. Then 
\eqref{eq:risk_dataconditional} is the false negative rate of the given prediction set.
\end{example}

We shall assume that the chosen loss function is \emph{monotone} in prediction set size, i.e.,
\[
\loss(\predictionset(X),Y)\ge \loss(\predictionset'(X),Y), \quad \forall \predictionset \subseteq \predictionset'\subseteq \mathcal{Y}
\]
and \emph{bounded} in $[0,B]$.

We can also consider the risk under known distribution shifts. That is, when calibration data is drawn as $\samples[n]\sim P$ but the test point is drawn from different distribution $( X,Y ) \sim P^*$ so that the test-time risk is
\begin{equation}
\E_{P^*}[\loss\left(\predictionset_{n}(X),Y\right) \: | \: \calibration[n] ].
\label{eq:risk_dataconditional_shift}
\end{equation}
It is assumed that $P^*$ is absolutely continuous with respect to $P$ and the Radon–Nikodym derivative
\begin{equation}
\omega(x,y) = \frac{\mathrm{d} P^*(x,y)}{\mathrm{d} P(x,y)}
\label{eq:importanceweight}
\end{equation}
is a known importance weight.

\subsection{Conformal Risk Control}

We consider a family of prediction sets $\predictionset_\threshold:\mathcal{X}\to\mathcal{P}(\mathcal{Y})$ indexed by $\threshold\in \Threshold$, where $\Threshold$ is totally ordered and the family is nested such that
\[
\predictionset_{\threshold}(X)\subseteq\predictionset_{\threshold'}(X), \; \quad \threshold \leq \threshold'
\]
and all $X$. It is assumed that there exists a  $\thresholdmax$ such that
$\loss(\predictionset_{\thresholdmax}(X),Y)\leq \alpha$, for all $\sample$. Note that $\predictionset_{\thresholdmax}(X)$ may equal  $\mathcal{Y}$ and is therefore uninformative. Throughout, $\predictionset_\threshold(X)$ denotes a generic prediction set parameterized by $\threshold$, and $\predictionset_{\threshold_n}(X)$ denotes the prediction set obtained by setting $\threshold = \threshold_n$ using calibration data of size $n$; we use $\predictionset_n(X)$ as shorthand for $\predictionset_{\threshold_n}(X)$ when the dependence on $\threshold_n$ is clear from the context.

\begin{example}
In regression problems, we may use score function, e.g., $s(X,y) = |y - f(X)|$ given a trained model $f(X)$, and then set $\predictionset_\threshold(X)=\{ y : s(X,y) \leq \threshold \}$.
\end{example}
\begin{example}
In multilabel classification problems, where $f:\mathcal{X}\to[0,1]^K$, we may use $\predictionset_\threshold(X)=\{ k : f_k(X) \ge 1-\threshold \}$, for $\threshold \in [0,1]$.
\end{example}

In conformal prediction and risk control, we specify a \emph{tolerated} risk level $\alpha$ and set a data-dependent threshold $\threshold_n$ using the calibration set $\calibration[n]=\samples[n]$. Specifically, choose the smallest $\threshold$ subject to a constraint:
\begin{equation}
\boxed{\threshold_n=\inf \left\{ \lambda : \Risk_n(\lambda) \le \alpha-\correction_n \right\},}
\label{eq:correctionterm}
\end{equation}
where $\Risk_n(\lambda)=\frac{1}{n}\sum_{i=1}^n \loss\left(\predictionset_{\threshold}(X_i),Y_i\right)$ is the empirical risk and $\correction_n$ is any \emph{correction} term. Intuitively,  the correction term $\correction_n$ acts as a safety margin. It inflates the threshold to account for statistical uncertainty in the empirical risk estimate.

\begin{example}
When the risk is the  miscoverage rate, the standard correction term is \citep{angelopoulos2023conformal}:
\begin{equation}
    \correction_n = \frac{\lceil (1-\alpha)(n+1)\rceil}{n} - (1-\alpha).
\label{eq:standardcorrection}
\end{equation}
This choice only ensures that the miscoverage rate of  prediction sets is below $\alpha$ \emph{when averaged over all $n$-sized calibration sets}, i.e., marginal risk control: $\E[\mathbb{P}(Y \not \in \predictionset_{\threshold_n}(X)| \calibration[n]) ]  \leq \alpha$.\\
More recently,  \citet{duchi2025sampleconditional} showed that the correction term
\begin{equation}
    \correction_n= \frac{4 \log \frac{1}{\delta}}{3n}
    + \sqrt{\Big(\frac{4}{3 n} \log \frac{1}{\delta}\Big)^2 +
      \frac{2 \alpha(1 - \alpha)}{n}
      \log \frac{1}{\delta}},
\label{eq:duchi}
\end{equation}
where $\delta \in (0,1)$, ensures that  $\predictionset_{\threshold_n}(X)$ has a sample-conditional miscoverage rate no greater than $\alpha$ with probability at least $1 - \delta$, i.e., $\mathbb{P}\big( \: \mathbb{P}(Y \notin \predictionset_{\threshold_n}(X) \mid \calibration[n]) \leq \alpha  \: \big) \geq 1-\delta$.  This strengthens the miscoverage control of the prediction set. 
\end{example}

\begin{remark}
If the set inside the infimum in \eqref{eq:correctionterm} is empty, we define $\threshold_n=\thresholdmax$. We assume that the empirical risk $\Risk_n(\threshold)$ is right-continuous in $\threshold$, 
meaning that for every $\threshold$ and $n$, 
\[
\lim_{\epsilon \to 0^+} \Risk_n(\threshold + \epsilon) = \Risk_n(\threshold).
\]
Right-continuity is standard in the conformal risk control literature \citep{angelopoulos2024conformal}. It guarantees that the infimum in \eqref{eq:correctionterm} is attained, i.e., $\Risk_n(\threshold_n)\le\alpha-\correction_n$. Without it, the selected threshold could violate the constraint. For the standard construction $\predictionset_\threshold(X)=\{y : s(X,y)\le\threshold\}$ with the miscoverage loss, $\loss(\predictionset_\threshold(X),Y)=\mathbf{1}\{s(X,Y)>\threshold\}$ is right-continuous in $\threshold$, so the assumption holds automatically.
\end{remark}

\subsection{Risk Control with Anytime Validity}

In this paper, we consider cumulative data collection processes such that $n$ in $\calibration[n]$ increases and yields a sequence of prediction sets $\{ \predictionset_{\threshold_n}(X) \}$. This motivates the following definition.

\begin{definition}
Let $\loss$ be a loss function and $\alpha$ a tolerated level of risk. For any given $\delta \in (0,1)$,  a method producing a sequence of prediction sets $\{ \predictionset_{\threshold_n} (X) \}$ is said to achieve \emph{anytime-valid risk control}, if 
\begin{equation}
\boxed{\mathbb{P}\Big(  \forall n \in \mathbb{N} : \; \E[\loss\left(\predictionset_{\threshold_n}(X),Y\right) \: | \: \calibration[n]] \: \leq \: \alpha\Big)  \; \geq \; 1-\delta.}
\label{eq:riskcontrol_dataconditional_anytime}
\end{equation}
That is, with a probability of at least $1-\delta$, we have a sequence of prediction sets for which \emph{all} the risks are no greater than the tolerance $\alpha$. Additionally, we say that the risk control is \emph{asymptotically tight} if the risks approach the specified level $\alpha$ from below.
\end{definition}

In \eqref{eq:riskcontrol_dataconditional_anytime}, the inner expectation is taken over the test point $(X,Y)\sim P$ (or $(X,Y)\sim P^*$ under distribution shift) with the realized calibration set $\calibration[n]$ held fixed, whereas the outer probability is taken over the entire infinite sequence of calibration data $(X_1,Y_1),(X_2,Y_2),\dots$. A similar notion of anytime-valid risk control is considered by \citet{xu2024active}.

The main question we aim to address is how to construct a \emph{sequence} of prediction sets $\{ \predictionset_{\threshold_n}(X) \}$ using \eqref{eq:correctionterm} that achieves anytime valid risk control and is asymptotically tight.

\section{Related Work}

\textbf{Conformal prediction and risk control.} Conformal prediction has had a significant impact in machine learning and overviews are provided in \citep{vovk2005algorithmic,angelopoulos2023conformal,angelopoulos2025theoreticalfoundationsconformalprediction}. \citet{pmlr-v25-vovk12} studied miscoverage rate of $\predictionset_{\lambda_n}(X)$ when \emph{conditioned} on $\calibration[n]$ and a distributional characterization is provided in \citep{hulsman2022distribution,MARQUESF2025110350}, which can be used to improve miscoverage control \citep{zwart2025probabilisticconformalcoverageguarantees} as alternative to \eqref{eq:duchi}. \citet{duchi2025sampleconditional} showed that using \eqref{eq:duchi} when the scores are pairwise distinct also yields a lower bound,
  \[
    \alpha - \frac{1}{n}-2\gamma_n
    \le \mathbb{P}(Y_{n + 1} \notin \predictionset_n(X_{n+1}) \mid \calibration[n])
    \le \alpha,
  \]
that holds with probability of at least $1-2 \delta$. Miscoverage control was generalized to marginal and sample-conditional risk control in \citep{angelopoulos2024conformal} and \citep{RCPS10.1145/3478535}, 
respectively. Sample-conditional miscoverage control under distribution shifts was recently achieved in \citep{pournaderi2026trainingconditional}.  Compared to \citet{pournaderi2026trainingconditional}, who also consider distribution shift, our guarantees are anytime-valid, apply to general bounded loss functions, and incorporate the importance weights explicitly in the correction term, resulting in bounds of a different nature. The two approaches rely on different proof techniques, a weighted DKW inequality versus sub-gamma concentration. As a special case of our framework, fixed-time sample-conditional bounds under covariate shift can be obtained, and we provide explicit instances in Appendix~\ref{app:fixedtime}.

\textbf{Anytime-valid concentration bounds.} A large body of work on sample-conditional conformal prediction under the i.i.d.\ setting relies on classical concentration inequalities, see, for instance \citep{pmlr-v25-vovk12, duchi2025sampleconditional, RCPS10.1145/3478535}. While most prior work relies on concentration inequalities for a fixed sample size, such tools are not always well suited to sequential settings. Anytime-valid concentration bounds, which hold uniformly over time, provide a natural alternative and have been studied for several decades, dating back to at least \citep{darlingd9fc6385-8ffb-31ca-a197-d0221f3b67d5}. These bounds have been refined and extended over the years \citep{Howard2018TimeuniformNN, howard2022sequentialestimation, 10.1093/jrsssb/qkad009,offpolicy10.1145/3643693}. In certain settings, asymptotic rates on the order of $O\left(\sqrt{n^{-1} \log \log n}\right)$ are achieved, and this asymptotic rate is known to be optimal \citep{pmlr-v247-duchi24a}. In this paper, we adapt the quantile construction of \citet{howard2022sequentialestimation} to the conformal risk control framework of \citet{angelopoulos2024conformal}, obtaining anytime-valid guarantees via sub-gamma concentration arguments following \citet{Howard2018TimeuniformNN,Howard_2020}. What is new is the reduction to a mean-zero bounded process at an oracle threshold, the resulting closed-form correction, the tightness and rate-optimality results and the importance-weighted extension.

\textbf{Sequential conformal prediction and risk control.} A related line of work studies conformal prediction in sequential settings, where coverage is guaranteed on average over time \citep{gibbs2021adaptive}. Most closely related is \citet{ge2025stochastic}, who obtain a guarantee of the form \eqref{eq:riskcontrol_dataconditional_anytime}, but only up to a fixed, known horizon $T$, as their focus is on regret minimization, whereas we require no pre-specified horizon. \citet{xu2024active} achieve anytime-valid risk control, also for actively labeled data, by running a betting-based e-process for each candidate threshold. Their thresholds are defined implicitly through the wealth processes, and their analysis concerns the growth rate of these processes, so how fast the risk approaches $\alpha$ is not explicit. In contrast, our correction $\gamma_n$ is closed-form and plugs directly into~\eqref{eq:correctionterm}, so this rate is explicit (Proposition~\ref{prop:lowerbound}) and of optimal order among corrections of this form (Proposition~\ref{prop:rateoptimal}). Our construction also extends to covariate shift with known likelihood ratios. \citet{gauthier2025evalues} use e-values to construct prediction sets that are simultaneously valid across sequentially arriving data batches, rather than uniformly valid as the calibration set grows within a single batch.

\section{Main Results}
\label{sec:main}

We now turn to setting the correction term in \eqref{eq:correctionterm} such that the resulting sequence $\{ \predictionset_{\threshold_n}(X) \}$ achieves anytime-valid risk control and is asymptotically tight. The threshold sequence $\{ \threshold_n \}$ will be non-increasing and, when necessary, we use a running minimum. Due to space limits, the proofs are relegated to Appendix~\ref{sec:derivations}.

\subsection{Anytime-Valid Conformal Risk Control}
\textbf{Intuition.} It suffices to control the gap between $\Risk_n$ and the population risk $\Risk(\threshold)=\E[\loss(\predictionset_\threshold(X),Y)]$ at the deterministic \emph{oracle threshold} $\hat\threshold(\alpha)=\inf\{\threshold:\Risk(\threshold)\le\alpha\}$. If $\Risk_n(\hat\threshold(\alpha))>\alpha-\correction_n$, then $\threshold_n\ge\hat\threshold(\alpha)$ by monotonicity, so the risk of $\predictionset_{\threshold_n}(X)$ is at most $\alpha$. Thus, $\correction_n$ only needs to be a deviation bound, valid simultaneously for all $n$, for an average of $n$ i.i.d.\ losses in $[0,B]$ with mean $\alpha$ (if $\Risk$ is continuous at $\hat\threshold(\alpha)$; see Appendix~\ref{sec:derivations} otherwise), whose accumulated variance is at most $v=\alpha(B-\alpha)n$.
 We let $m\in\mathbb{N}$, $c\geq 0$ and $\delta\in(0,1)$. For notational convenience, define the following two functions
\begin{align*}
    h_{m,\delta}(v) \;&=\; 2\,\log\left(\log_2\left( \frac{\max\{v,m\}}{m}\right)+1\right)+\; \log(\pi^2/(6\delta)),
\end{align*}
and
\begin{align*}
f_{c,m,\delta}(v) \;=&\; 1.44\,\sqrt{v\,h_{m,\delta}(v)} \;+\; 2.42\, c\,h_{m,\delta}(v).
\label{eq:radii_function}
\end{align*}

The function $f_{c,m,\delta}$ is a Bernstein-type bound with a variance term and a range term with scale $c$, where $h_{m,\delta}$ replaces the factor $\log(1/\delta)$ of a fixed-time bound.
\begin{restatable}{theorem}{mainriskcontrol}
\label{thm:main_riskcontrol}

    Let $\loss$ be an arbitrary loss function bounded in $[0,B]$, monotone in $\threshold$, and right-continuous in $\threshold$. Let $c_\alpha=\max\{(2\alpha-B)/3,0\}$. Construct the prediction sets according to \eqref{eq:correctionterm} and set the correction term as
 \begin{equation}
    \correction_n=  f_{c_\alpha,\alpha(B-\alpha)m^*,\delta}\big(\alpha (B-\alpha)n\big)/n, 
\label{eq:correctionterm_riskcontrol}
 \end{equation}
 where 
\[
m^* \;=\; \min\Bigl\{\, m' \in \mathbb{N} : f_{c_\alpha,\alpha(B-\alpha)m',\delta}\big(\alpha (B-\alpha)m'\big)/m' \le \alpha \,\Bigr\}.
\]
Then the resulting sequence of prediction sets $\{ \predictionset_{\lambda_n}(X) \}$ achieves anytime-valid risk control \eqref{eq:riskcontrol_dataconditional_anytime}.
\end{restatable}

\begin{corollary}\label{cor:anytimeconfpred}
In conformal prediction, where $\loss$ is the miscoverage loss, anytime valid miscoverage rate control is achieved by setting $B=1$ in \eqref{eq:correctionterm_riskcontrol}.
\end{corollary}
The proofs, provided in Appendix~\ref{sec:derivations}, are based on constructing suitable sub-gamma processes, applying a time-uniform concentration bound, and carefully handling potential discontinuities in the population risk.
\begin{remark}
A smaller correction term $\correction_n$ yields smaller prediction sets (i.e. higher efficiency). Note, however, that smaller prediction sets do not automatically imply valid risk control, so efficiency is meaningful only when paired with a validity guarantee. We plot the correction terms \eqref{eq:standardcorrection}, \eqref{eq:duchi}, and \eqref{eq:correctionterm_riskcontrol} against $n$ in Figure~\ref{fig:correction}, illustrating the cost to strengthen marginal to fixed-time and anytime validity in terms of  sample efficiency (visible as the vertical gap between the curves).
\end{remark}

\begin{figure*}[h!]
    \centering
    \begin{subfigure}[b]{0.45\textwidth}
        \centering
        \includegraphics[width=\textwidth]{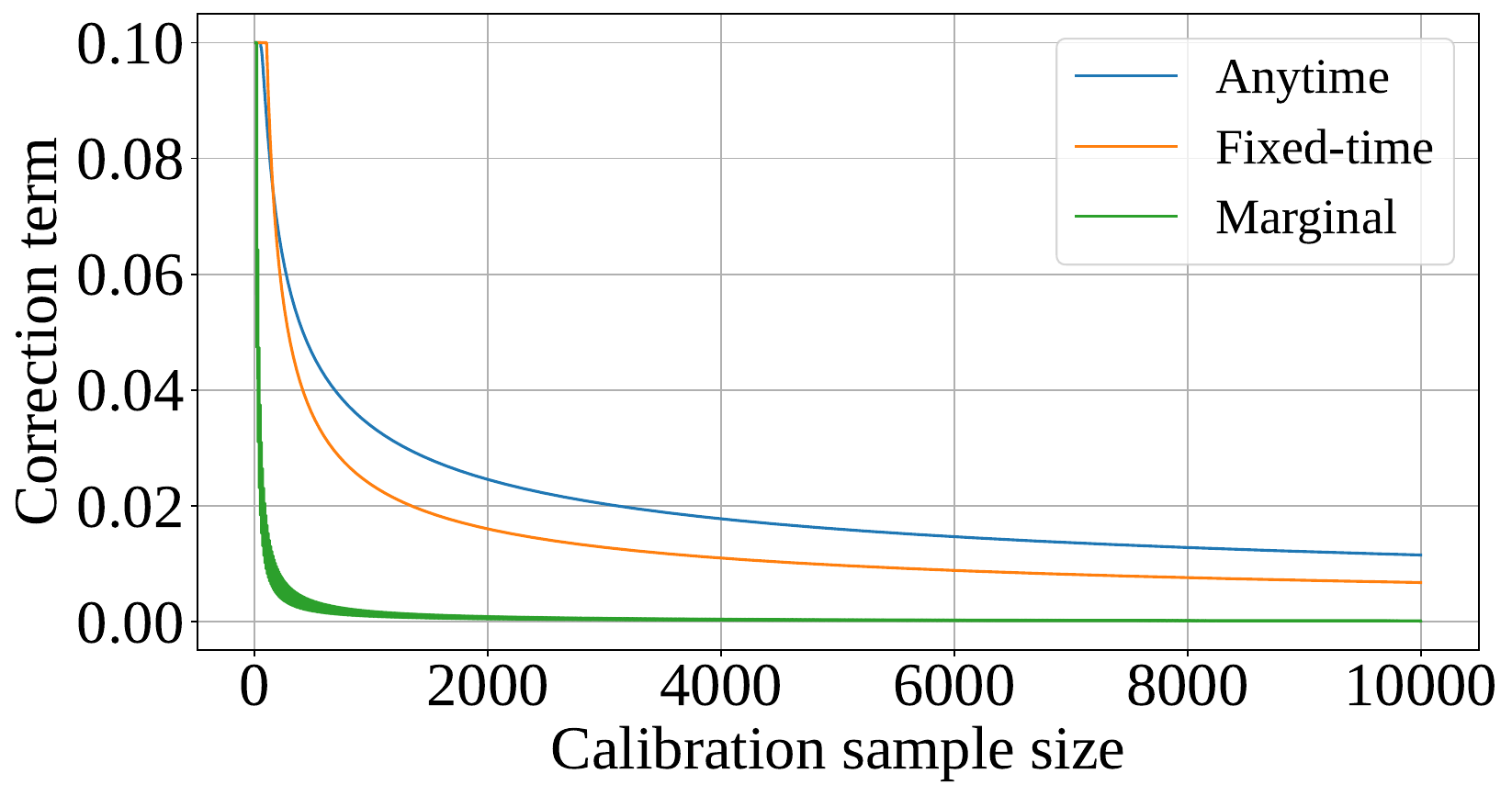}
        \caption{$\alpha=10 \%,\delta=10 \%$}
        \label{fig:smalldelta}
    \end{subfigure}
    \hfill
    \begin{subfigure}[b]{0.45\textwidth}
        \centering
        \includegraphics[width=\textwidth]{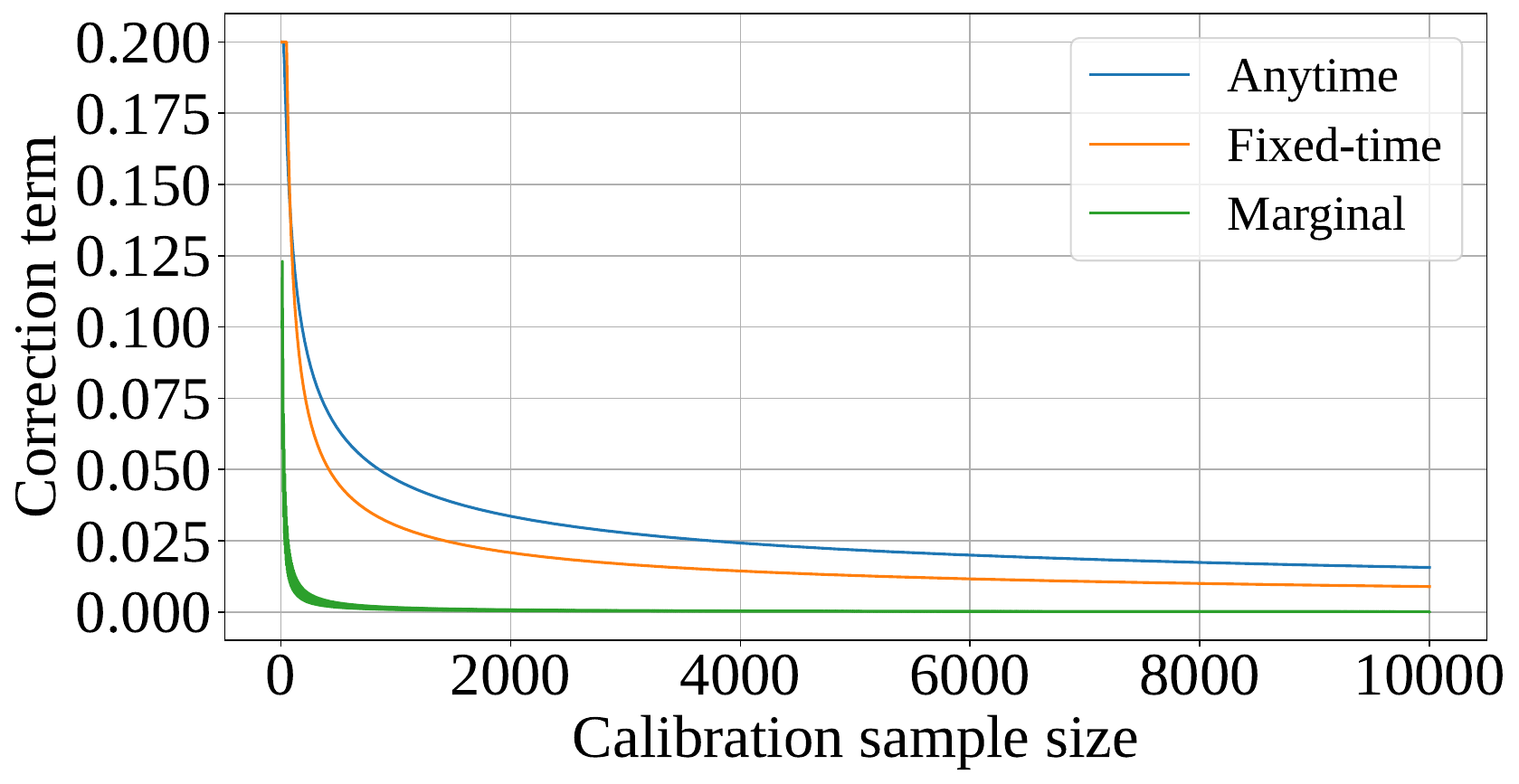}
        \caption{$\alpha=20 \%,\delta=10 \%$}
        \label{fig:bigdelta}
    \end{subfigure}
    
    \caption{Correction term $\gamma_n$ in \eqref{eq:correctionterm} versus $n$. Using \eqref{eq:standardcorrection}, \eqref{eq:duchi}, and \eqref{eq:correctionterm_riskcontrol}, guarantees marginal, fixed-time and anytime validity, respectively. }
    \label{fig:correction}
\end{figure*}

\begin{remark}    
\label{rem:smallncases}
    For $n<m^*$, we have $\correction_n>\alpha$. Then \eqref{eq:correctionterm} yields the infimum of an empty set which, by definition, results in $\threshold_n=\thresholdmax$ and an uninformative prediction set at that sample $n$. The burn-in $m^*$ is a design choice. The proof of Theorem~\ref{thm:main_riskcontrol} remains valid if $m^*$ in \eqref{eq:correctionterm_riskcontrol} is replaced by any $m\in\mathbb{N}$ fixed in advance, with $\alpha(B-\alpha)n$ replaced by $\alpha(B-\alpha)\max\{n,m\}$. No choice of $m$ yields informative sets before $n=m^*$, and $m<m^*$ only increases $\correction_n$ for $n\ge m^*$. Choosing $m>m^*$ delays informative sets but tightens them for all $n\ge m$, as less of the error budget $\delta$ is spent on small $n$. Hence, there is no uniformly optimal choice. If prediction sets are only needed from some calibration size $n_0\ge m^*$ onward, $m=n_0$ is a natural choice.
\end{remark}
\begin{remark}
    \label{rem:altconstructions}
    There are alternative ways to construct valid sequences of correction terms. For further details, we refer to Remark~\ref{app:fconstruction}.
\end{remark}

Assuming that the empirical risk $\Risk_n$ does not exhibit large jumps, we can show that the choice of correction terms $\correction_n$ is tight up to a controlled error. For a right-continuous function $l$, we define the jump at $\threshold$ by $J(l,\lambda) =\lim_{\epsilon \to 0^+}l(\threshold-\epsilon) - l(\threshold)$.

\begin{restatable}{proposition}{lowerbound}
\label{prop:lowerbound}
    In the setting of Theorem~\ref{thm:main_riskcontrol}, let $\{\threshold_n\}$ be non-increasing (using a running minimum if necessary) and assume that $\alpha\le B/2$ and $\alpha\le\sup_{\threshold\in\Threshold}\Risk(\threshold)$, where $\Risk(\threshold)=\E[\loss(\predictionset_\threshold(X),Y)]$. Furthermore, let $d_n=\sup_{\threshold\in \Threshold} J(\Risk_n,\lambda)$.
    If $d_n $ converges to $ 0$, then the sequence $\{ \predictionset_{\lambda_n}(X) \}$ is asymptotically tight with the risk approaching $\alpha$ at the rate $O\left(\max\left\{d_n,B\sqrt{n^{-1}\log\log n}\right\}\right)$. In particular, the sequence  $\{ \predictionset_{\lambda_n}(X) \}$ satisfies
  \begin{equation*}
\begin{split}
\mathbb{P}\big(  \forall n : \alpha-k_n\le
\E[\loss\left(\predictionset_{\threshold_n}(X),Y\right)| \calibration[n]] \leq \alpha\big)\geq 1-2\delta,
\end{split}
\end{equation*}

where $k_n = d_{g(n)} + \gamma_{g(n)} + \gamma'_{g(n)}(\delta_{\lfloor \log_2 n \rfloor+1})$, $g(n) = 2^{\lfloor \log_2 n \rfloor}$, $\delta_i = \delta/(i(i+1))$, and
\begin{equation*}
\gamma_n'(\delta) = \frac{B}{n} \log \frac{1}{\delta} + \sqrt{2 \frac{\alpha(B-\alpha)}{n}\log \frac{1}{\delta} + \left(\frac{B}{n} \log \frac{1}{\delta}\right)^2}.
\end{equation*}
\end{restatable}
\begin{remark}
    The condition $d_n \to 0$ is satisfied in several settings. For instance, if $\Risk_n$ is continuous, then $d_n=0$. Alternatively, let $L_i(\threshold)=\loss(\predictionset_\threshold(X_i),Y_i)$ and assume that for any $\threshold\in\Threshold$ and $i$, $\mathbb{P}\left(J(L_i, \threshold) > 0 \right) = 0$. Then $d_n \overset{\text{a.s.}}{\le} B/n$ (Lemma~1, \citealp{angelopoulos2024conformal}). This can be interpreted as follows: for any fixed $\threshold$, the function $L_i$ is almost surely continuous. In other words, $L_i$ might have discontinuities, but their locations are random rather than fixed. For example, in conformal prediction with almost surely distinct scores, every jump of $\Risk_n$ has size exactly $1/n$, so that $d_n\le1/n$. The condition may fail if the score distribution has atoms. The clipped score $\min\{|y-f(x)|,C\}$ places positive mass at $C$, in which case $d_n$ converges to a positive constant. Thus, the condition concerns the distribution of the jump locations rather than continuity of the loss or the score function. In practice, $d_n$ can be computed directly from the calibration data, since $\Risk_n$ is already evaluated when solving \eqref{eq:correctionterm}. Moreover, the proof only uses the size of the jump at the selected threshold, so $d_n$ may be replaced by the smaller, also computable quantity $J(\Risk_n,\threshold_n)$.
\end{remark}

\subsection{Anytime-Valid Risk Control under Distribution Shift}

We generalize Theorem~\ref{thm:main_riskcontrol} to the setting of conformal risk control under test-time distribution shift. The calibration data is drawn from a distribution $P$, which may differ from  distribution $P^*$ for the test data. The goal is to control the test-time risk \eqref{eq:risk_dataconditional_shift}, assuming $\omega(x,y)$ is given \eqref{eq:importanceweight}.

\begin{restatable}{theorem}{distshiftcontrol}
\label{thm:distshiftcontrol}
Let $\loss$ be an arbitrary loss function bounded in $[0,B]$, monotone in $\threshold$, and right-continuous in $\threshold$. Define $b_{m,\delta}(v)=1.44\sqrt{vh_{m,\delta}(v)}$, i.e., $b_{m,\delta}=f_{0,m,\delta}$. Define the importance weights $\omega_i := \omega(X_i,Y_i)$ using (\ref{eq:importanceweight}) and $W_n=\sum_{i=1}^n \omega_i^2$, and fix $m>0$ before observing the calibration data. 
Construct the prediction sets according to \eqref{eq:correctionterm}, using the importance-weighted loss $\omega(x,y)\,\loss(\predictionset_\threshold(x),y)$, i.e., with empirical risk $\Risk_n^\omega(\threshold)=\frac{1}{n}\sum_{i=1}^n\omega_i\,\loss(\predictionset_\threshold(X_i),Y_i)$, with the correction term as
    \[
    \correction_n=B\left(1-\frac{1}{n}\sum_{i=1}^n \omega_i \right)+ \frac{b_{m,\delta}(\max\{B^2W_n,m\})}{n}.
    \]
   Then the resulting sequence of prediction sets $\{ \predictionset_{\lambda_n}(X) \}$ achieves anytime-valid  risk control \eqref{eq:riskcontrol_dataconditional_anytime} under distribution shift \eqref{eq:risk_dataconditional_shift}. That is, with probability at least $1-\delta$, we obtain a sequence of prediction sets such that
     \[
    \E_{P^*}[\loss(\predictionset_{\threshold_n}(X),Y) \: | \: \calibration[n]] \: \leq \: \alpha
    \]
    holds for all $n$.
\end{restatable}
\begin{remark}
For $\alpha\le B/2$, we have $c_\alpha=0$, so Theorems~\ref{thm:main_riskcontrol} and~\ref{thm:distshiftcontrol} use the same boundary $b_{m,\delta}$ and differ mainly in the variance proxy. Without distribution shift, the losses at the oracle threshold lie in $[0,B]$ and have mean $\alpha$, so the relevant (lower) tail of their average is sub-Gaussian with the small variance $\alpha(B-\alpha)$. Under distribution shift, the weighted losses may be unbounded, and their variability is controlled through the observed weights, which leads to the variance proxy $B^2W_n$. For further technical details, we refer to Remark~\ref{app:c=0}.
\end{remark}

\begin{remark}
Since $P^*$ is absolutely continuous with respect to $P$, we have
$\mathbb{E}_P[\omega_i]=1$. By the strong law of large numbers, the term
$B\big(1-\frac{1}{n}\sum_{i=1}^n\omega_i\big)$ therefore tends to zero almost
surely, with no further assumption on the weights.
\end{remark}

\begin{restatable}{proposition}{distshiftcontrollower}
\label{prop:distshiftlower}
    In the distribution-shift setting of Theorem~\ref{thm:distshiftcontrol}, let $\{\threshold_n\}$ be non-increasing (using a running minimum if necessary) and assume that $\alpha\le\sup_{\threshold\in\Threshold}\E_{P^*}[\loss(\predictionset_\threshold(X),Y)]$. Let $d_n=\sup_{\threshold\in \Threshold} J(\Risk^{\omega}_n,\lambda)$. If $d_n $ converges to $0$ and $\sqrt{\frac{W_n\log\log W_n}{n^2}} \to 0$, where the latter holds almost surely
whenever $\mathbb{E}_P[\omega^2] < \infty$,
then the sequence $\{ \predictionset_{\lambda_n}(X) \}$ is asymptotically tight. In particular, the sequence  $\{ \predictionset_{\lambda_n}(X) \}$ satisfies
  \begin{equation*}
\begin{split}
\mathbb{P}\big(  \forall n : \alpha-k_n\le
\E_{P^*}[\loss\left(\predictionset_{\threshold_n}(X),Y\right)| \calibration[n]] \leq \alpha\big)\geq 1-2\delta,
\end{split}
\end{equation*}
where $k_n = d_{g(n)}+\correction_{g(n)}+\correction'_{g(n)}$, $g(n) = 2^{\lfloor \log_2 n \rfloor}$, and $\correction'_n=2^{j_n}/n-\correction_n$, where 
\[
j_n=\min\Big\{j\ge0:\ \frac{2^j}{n}\ge\correction_n+\frac{b_{m,\delta_{n,j}}(\max\{B^2W_n,m\})}{n}\Big\},\qquad \delta_{n,j}=\frac{\delta_{\lfloor\log_2n\rfloor+1}}{(j+1)(j+2)},
\]
and $\delta_i=\frac{\delta}{i(i+1)}$.

\end{restatable}

\subsection{Optimality of the Rate}
\label{sec:rate}

Proposition~\ref{prop:lowerbound} shows that the risks of our procedure converge to $\alpha$ at rate $O(\sqrt{n^{-1}\log\log n})$ when $d_n$ vanishes sufficiently fast. The correction term \eqref{eq:correctionterm_riskcontrol} is of order $\sqrt{n^{-1}\log\log n}$, and we now show that any such term in \eqref{eq:correctionterm} that vanishes faster than this rate violates the risk constraint for infinitely many $n$ almost surely.
\begin{proposition}[Rate optimality]
\label{prop:rateoptimal}
Let $\loss$ be the miscoverage loss, $\alpha\in(0,1)$, and let $\Threshold\subseteq\R$ be an interval. Suppose that the population risk $\Risk(\threshold)=\E[\loss(\predictionset_\threshold(X),Y)]$ is continuous at $\hat\threshold(\alpha)=\inf\{\threshold:\Risk(\threshold)\le\alpha\}$, which lies in the interior of $\Threshold$. Consider any procedure of the form \eqref{eq:correctionterm} with a deterministic correction sequence $\correction_n\ge0$ satisfying $\correction_n=o\big(\sqrt{n^{-1}\log\log n}\big)$. Then
\[
\mathbb{P}\Big(\forall n\in\mathbb{N}:\ \E[\loss(\predictionset_{\threshold_n}(X),Y)\mid\calibration[n]]\le\alpha\Big)=0.
\]
In particular, no such correction sequence achieves anytime-valid risk control \eqref{eq:riskcontrol_dataconditional_anytime} for any $\delta\in(0,1)$.
\end{proposition}

The proof, given in Appendix~\ref{sec:derivations}, relies on the law of the iterated logarithm. In particular, the standard correction \eqref{eq:standardcorrection} and the fixed-time correction \eqref{eq:duchi}, which decay as $O(n^{-1})$ and $O(n^{-1/2})$, respectively, cannot achieve anytime-valid risk control, regardless of $\delta$. Since the correction term \eqref{eq:correctionterm_riskcontrol} is of order $\sqrt{n^{-1}\log\log n}$, Theorem~\ref{thm:main_riskcontrol} and Propositions~\ref{prop:lowerbound} and~\ref{prop:rateoptimal} together show that our correction has the optimal rate.

\section{Numerical Experiments}
We illustrate the performance of the anytime-valid risk controlled sets in five examples below. 

\textbf{Synthetic Regression.}
\label{sec:synthetic}
The data are generated according to the linear model
\begin{equation*}
Y_i = 2 X_i + \epsilon_i, \quad \epsilon_i \sim \mathcal{N}(0,1), \quad X_i \sim \text{Uniform}(-3,3),
\end{equation*}
where $i \in \mathbb{N}$. We set the target miscoverage level to $\alpha= 5 \%$, $\delta= 10\%$, and let $\calibration[n]=\samples[n]$. Assume we have access to a (correctly specified) predictive model $f$. We use a simple score function $s(X,y) = |y - f(X)|$. Following Corollary~\ref{cor:anytimeconfpred}, we construct a sequence of prediction sets $\{ \predictionset_{\lambda_n}(X)\}$. For each $n \in \mathbb{N}$, we evaluate the corresponding miscoverage probability $\mathbb{P}(Y \not \in \predictionset_{\threshold_n}(X)| \calibration[n])$. 

Since the sequence $\{ \predictionset_{\lambda_n}(X)\}$ depends on the realized calibration sample $\calibration[n]$ the resulting miscoverage probabilities are random. 
To characterize this variability, we repeat the entire data collection and prediction-set construction procedure over multiple independent runs, and for each run plot the miscoverage probability as a function of $n$ in Figure~\ref{fig:1b}. 

For comparison, we also construct prediction sets using the standard correction term in \eqref{eq:standardcorrection} and the fixed-time valid correction term in \eqref{eq:duchi}. The latter shows a clearer structure in which the miscoverage generally approaches $\alpha$ from below. While the probability of achieving miscoverage below $\alpha$ is at least $1 - \delta$ for any fixed $n$, the procedure is not uniformly valid over time, and consequently many runs still exceed the miscoverage threshold at some sample size $n$.

\begin{wrapfigure}{r}{0.48\textwidth}
    \centering
    \vspace{-13pt}
    \includegraphics[width=\linewidth]{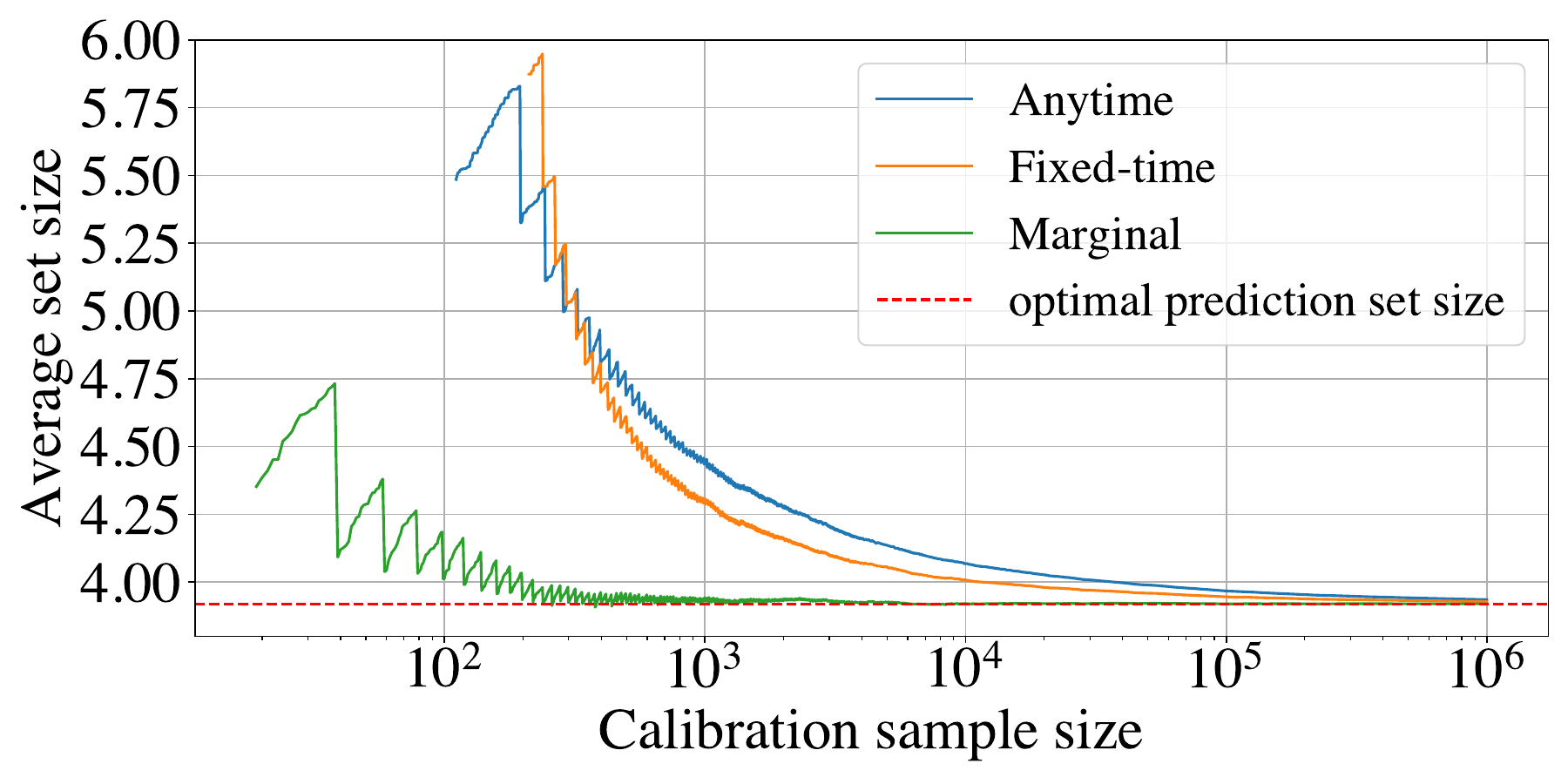}\vspace{-5pt}
    \caption{Average prediction set sizes across methods.}
    \label{fig:setsize51}
\end{wrapfigure}
Indeed, by the law of the iterated logarithm, using either correction term in \eqref{eq:standardcorrection} or \eqref{eq:duchi} yields a sequence of risks that almost surely exceeds the tolerance level $\alpha$ infinitely often (see, e.g., Theorem~8.5.2 in \cite{Durrett2019}). In finite simulations, this behavior may not be observed, since only a finite time horizon is considered. In contrast, the anytime-valid procedure maintains miscoverage below $\alpha$ across runs, with most sequences remaining green, and the realized miscoverage approaching $\alpha$ from below as $n$ increases. 
Figure~\ref{fig:setsize51} shows the corresponding average prediction set sizes.

\textbf{Synthetic Regression under Distribution Shift.} We illustrate Theorem~\ref{thm:distshiftcontrol} using the example in Section~2 of \cite{SHIMODAIRA2000227}. We assume that the data are generated according to 
\[
Y_i=-X_i+X_i^3+\epsilon_i, \quad \epsilon_i\sim \mathcal{N}(0,0.3^2),
\]
where $X_i\sim \mathcal{N}(0.5,0.5^2)$.
A linear model $f$ is fitted using $n=1000$ held-out data points. We assume that test points come from a different distribution $P^*$, with which $X \sim \mathcal{N}(0,0.3^2)$. Thus $\omega(x,y) = p^*(x)/p(x)$ is a ratio of Gaussian densities. The aim is to control the miscoverage rate
\[
\E_{\sample\sim P^*}\left[\mathbf{1}\{Y \notin\predictionset_{\threshold_n}(X)\}|\calibration[n]\right].
\]
We set $\alpha=10 \%, \delta=10\%$, and perform a Monte Carlo simulation as in the synthetic experiment above. The resulting miscoverage behavior is illustrated in Figure~\ref{fig:distshift}. For comparison, we also plot the resulting miscoverage in the absence of distribution shift, $\omega(x,y) \equiv 1$, using Corollary~\ref{cor:anytimeconfpred}. Figure~\ref{fig:distshift_noshift} shows the convergence towards $\alpha$ is substantially faster in this case.

\begin{figure*}[h!]
    \centering
    \begin{subfigure}[b]{0.45\textwidth}
        \centering
        \includegraphics[width=\textwidth]{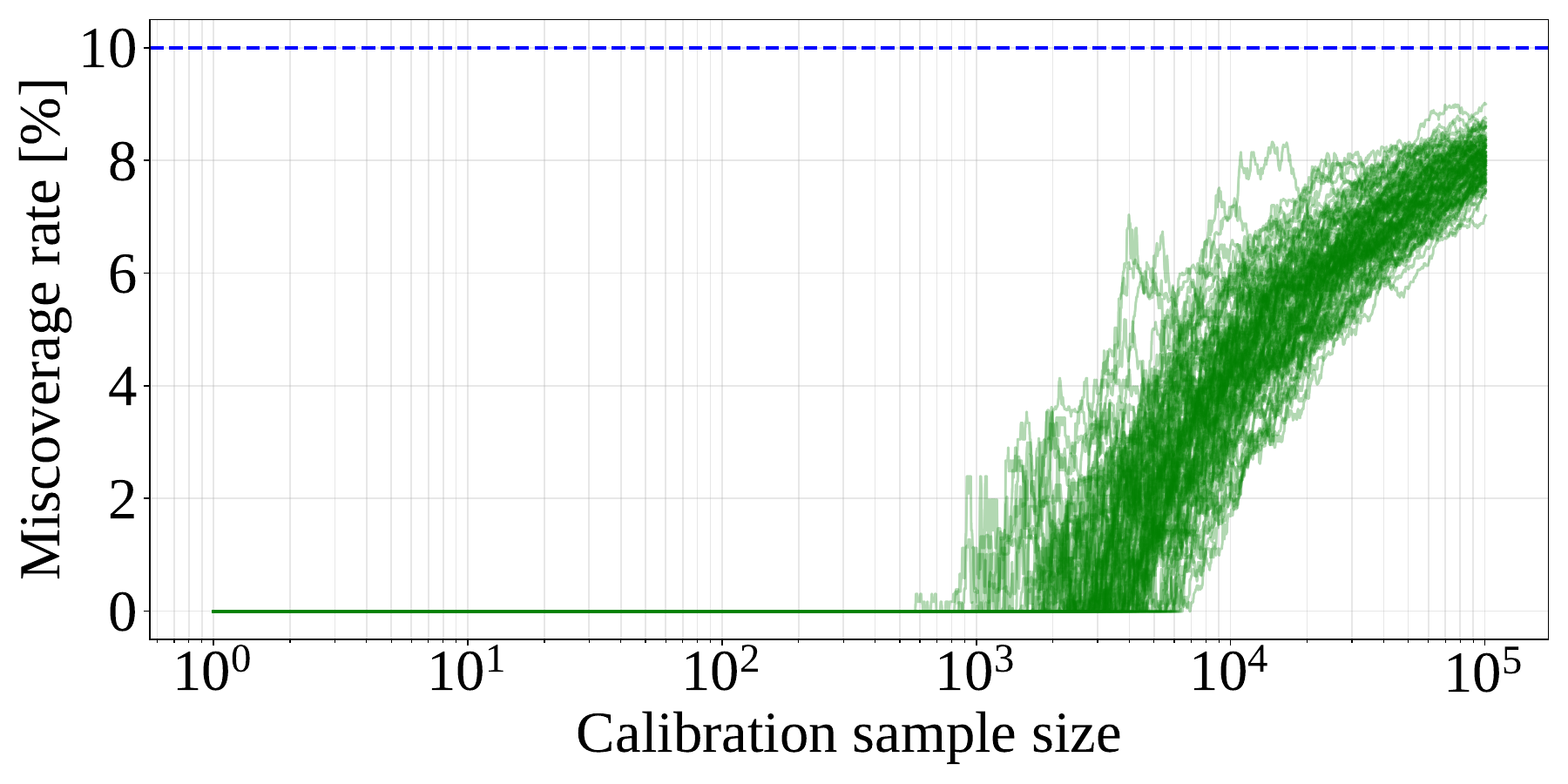}
        \caption{With distribution shift.}
        \label{fig:distshift}
    \end{subfigure}
    \hfill
    \begin{subfigure}[b]{0.45\textwidth}
        \centering
        \includegraphics[width=\textwidth]{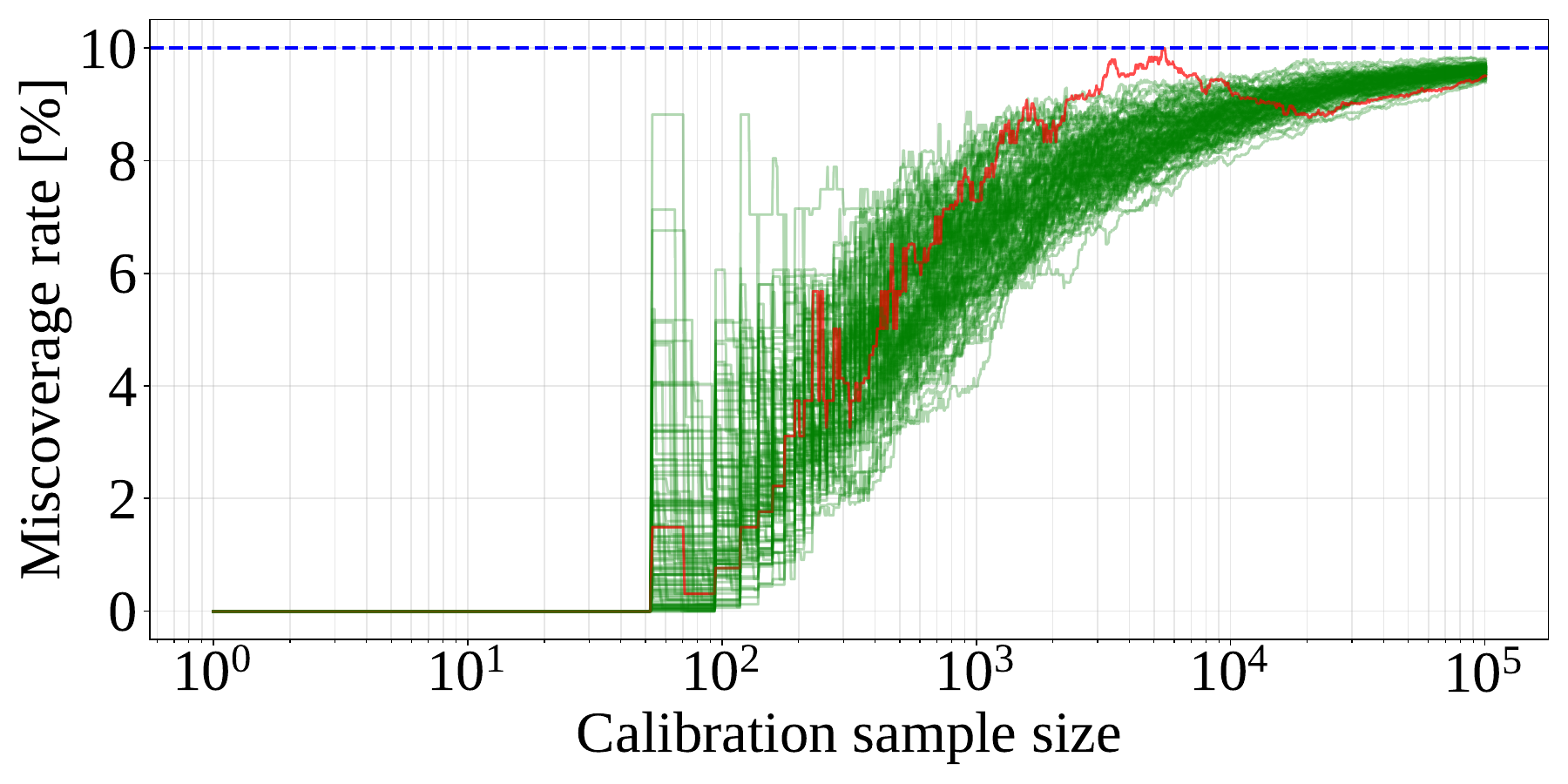}
        \caption{Without distribution shift ($\omega \equiv 1$).}
        \label{fig:distshift_noshift}
    \end{subfigure}
    \caption{Miscoverage rates of prediction sets $\{ \predictionset_n(X) \}$ versus calibration sample size $n$, with $\alpha=10\%$, $\delta=10\%$. }
    \label{fig:distshift_both}
\end{figure*}

\textbf{Image Classification.} We apply Corollary~\ref{cor:anytimeconfpred} to image classification using the data set ImageNet-1K, which consists of $K=1000$ different classes. We  use the base implementation from \cite{angelopoulos2022gentleintroductionconformalprediction}.
Our goal is, given an image, to output a set of classes containing the true class with probability of at least $1-\alpha$. Let $\mathcal{X}$ denote the image space and $\mathcal{Y} = \{1, \dots, K\}$ the set of classes. For an image $x$, the model gives estimated class probabilities
\[
f(x)\in [0,1]^K, \quad \sum_{i=1}^K f_i(x)=1.
\]
For a sample $\sample$, where $X\in \mathcal{X}$ is the image and $Y \in \mathcal{Y}$ its true class, we define the score function as $s\sample=1-f_Y(X)$. We target miscoverage at level $\alpha = 10\%$ with $\delta = 10\%$, and compare the prediction sets obtained using Corollary~\ref{cor:anytimeconfpred}, the fixed-time correction~\eqref{eq:duchi}, and the standard correction~\eqref{eq:standardcorrection}. Using a held-out evaluation set, we estimate the expected sizes of these prediction sets and plot the results in Figure~\ref{fig:predictionsetsizecomp}. They illustrate that anytime-validity can be gained at equivalent prediction set sizes even with a small increase of sample size $n$. We compare the empirical miscoverage in Figure~\ref{fig:missccomp}.

\begin{figure*}[h!]
    \centering
    \begin{subfigure}[b]{0.45\textwidth}
        \centering
        \includegraphics[width=\textwidth]{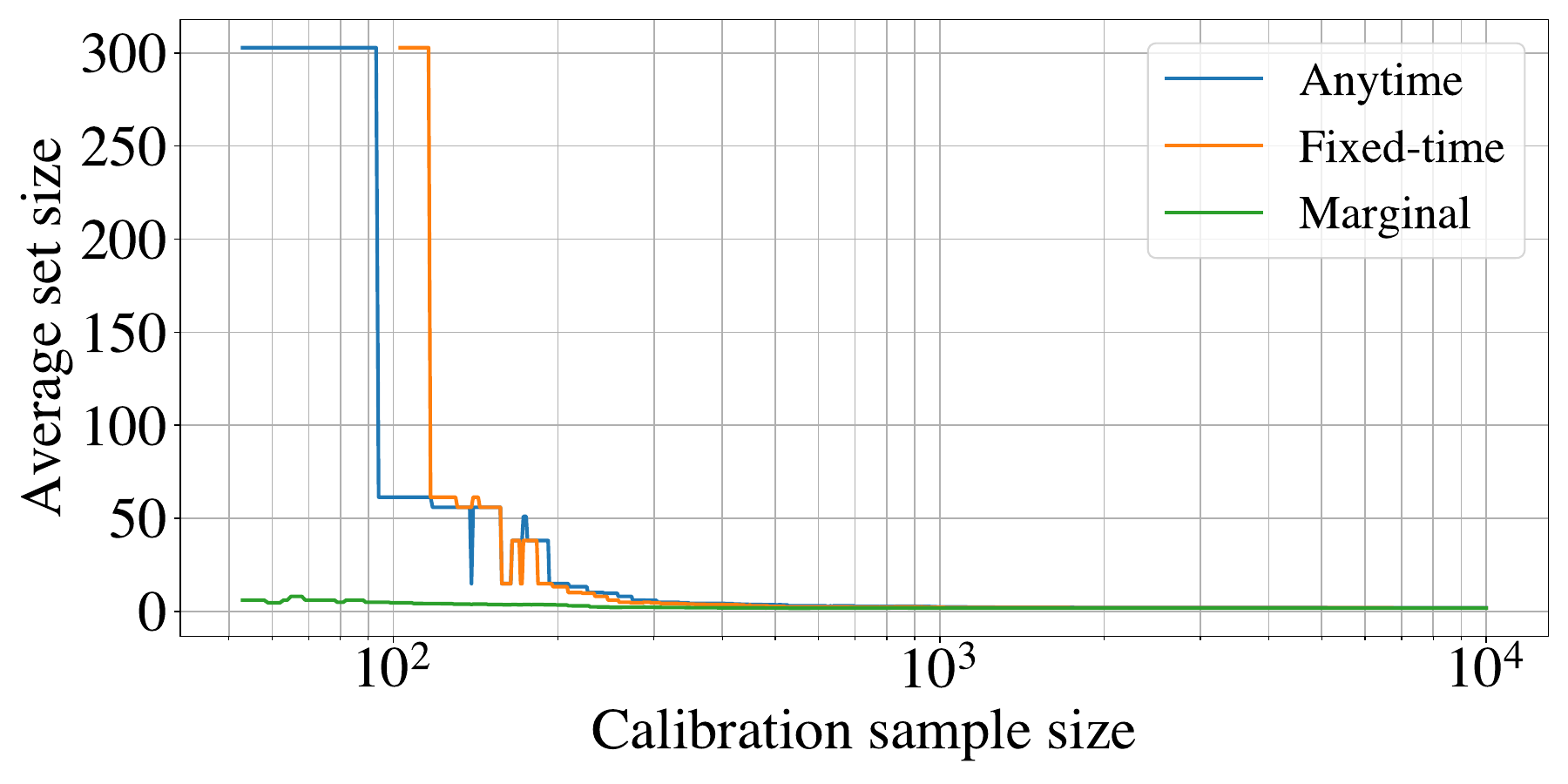}
        \caption{}
        \label{fig:predictionsetsizecomp}
    \end{subfigure}
    \hfill
    \begin{subfigure}[b]{0.45\textwidth}
        \centering
        \includegraphics[width=\textwidth]{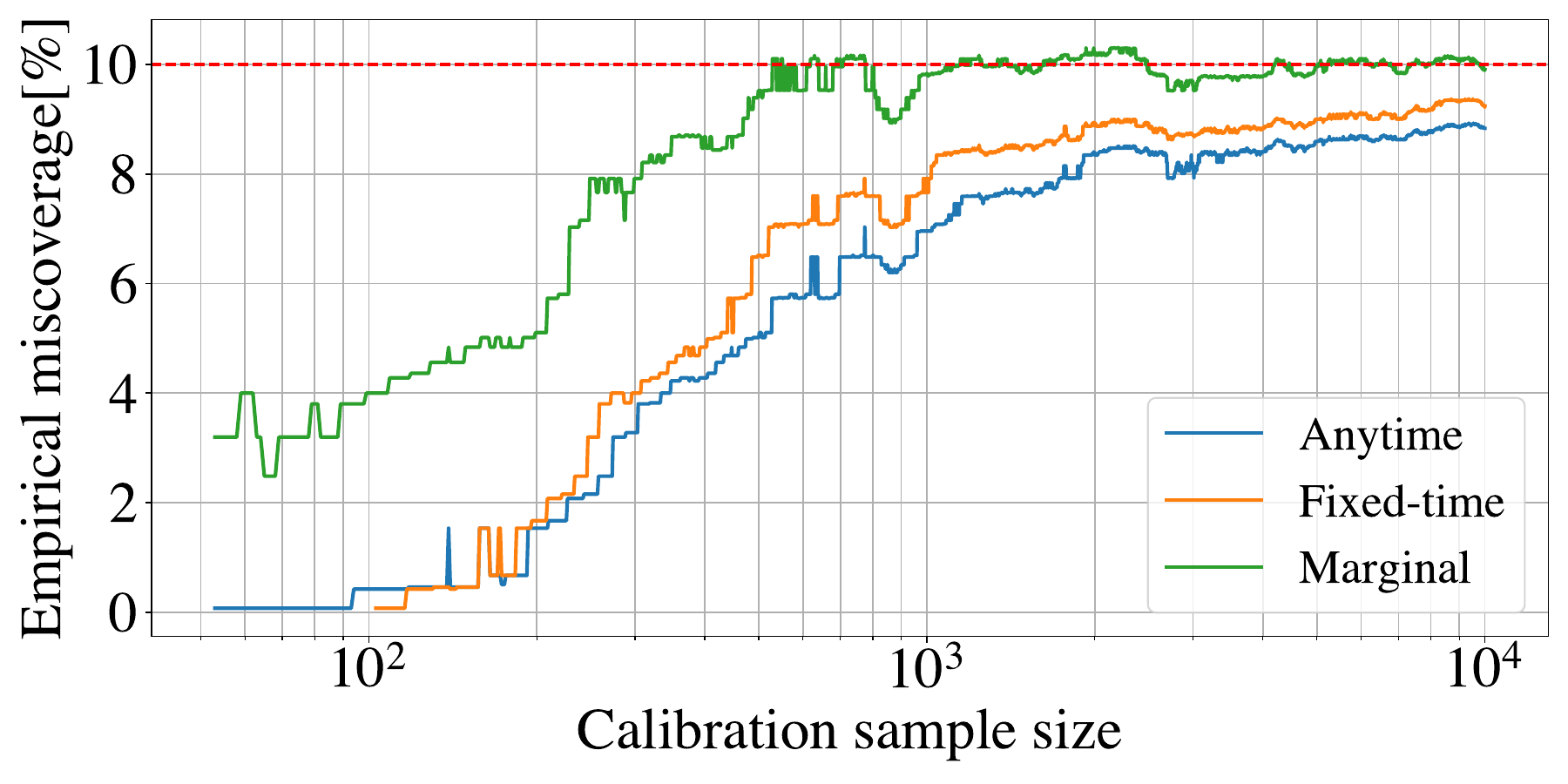}
        \caption{}
        \label{fig:missccomp}
    \end{subfigure}
    \caption{Comparison of anytime-valid conformal prediction using Corollary~\ref{cor:anytimeconfpred}, the fixed-time method \eqref{eq:duchi}, and marginal conformal prediction \eqref{eq:standardcorrection} for ImageNet-1K classification with $\alpha = 10\%$ and $\delta = 10\%$. }
    \label{fig:distshift_both}
\end{figure*}

\textbf{Synthetic False Negative Rate Control.} 
We construct a synthetic multilabel classification problem with $K = 3$ labels and $D = 3$-dimensional input features. The covariates are generated as $X_i \sim \mathcal{N}(0, I_D)$. For each instance $i$, the label vector $Y_i \in \{0,1\}^K$ has conditionally independent components given $X_i$, with
\[
Y_{ik} \sim \mathrm{Bernoulli}\!\left(\sigma\!\left(e_k^\top X_i - 1\right)\right), \quad k = 1, \dots, K,
\]
where $\sigma(\cdot)$ denotes the logistic sigmoid function and $e_k$ is the $k$-th standard basis vector in $\mathbb{R}^D$.

We assume access to a perfectly specified model $f$ that outputs the true class probabilities, i.e., $f_k(X) = \mathbb{P}(Y_k = 1 \mid X)$. Prediction sets are constructed via thresholding:
\[
\mathcal{C}_\lambda(X) = \{\, k \in \{1,\dots,K\} : f_k(X) \geq 1 - \lambda \,\}.
\]

We target FNR control from Example~\ref{ex:FNR} at level $\alpha=10\%$ with $\delta=10\%$, 
and repeat the same experiment as in the synthetic 
regression experiment above. Figure~\ref{fig:fnr} shows the realized FNR of the sequence of prediction 
sets constructed using 
Theorem~\ref{thm:main_riskcontrol}, while Figure~\ref{fig:fnrwsr} shows the corresponding results using the method of \citet{RCPS10.1145/3478535} (Proposition~5 therein), which applies the fixed-time Waudby-Smith--Ramdas upper confidence bound \citep{10.1093/jrsssb/qkad009} to achieve sample-conditional risk control at a single calibration 
size $n$.

\begin{figure*}[h!]
    \centering
    \begin{subfigure}[b]{0.45\textwidth}
        \centering
        \includegraphics[width=\textwidth]{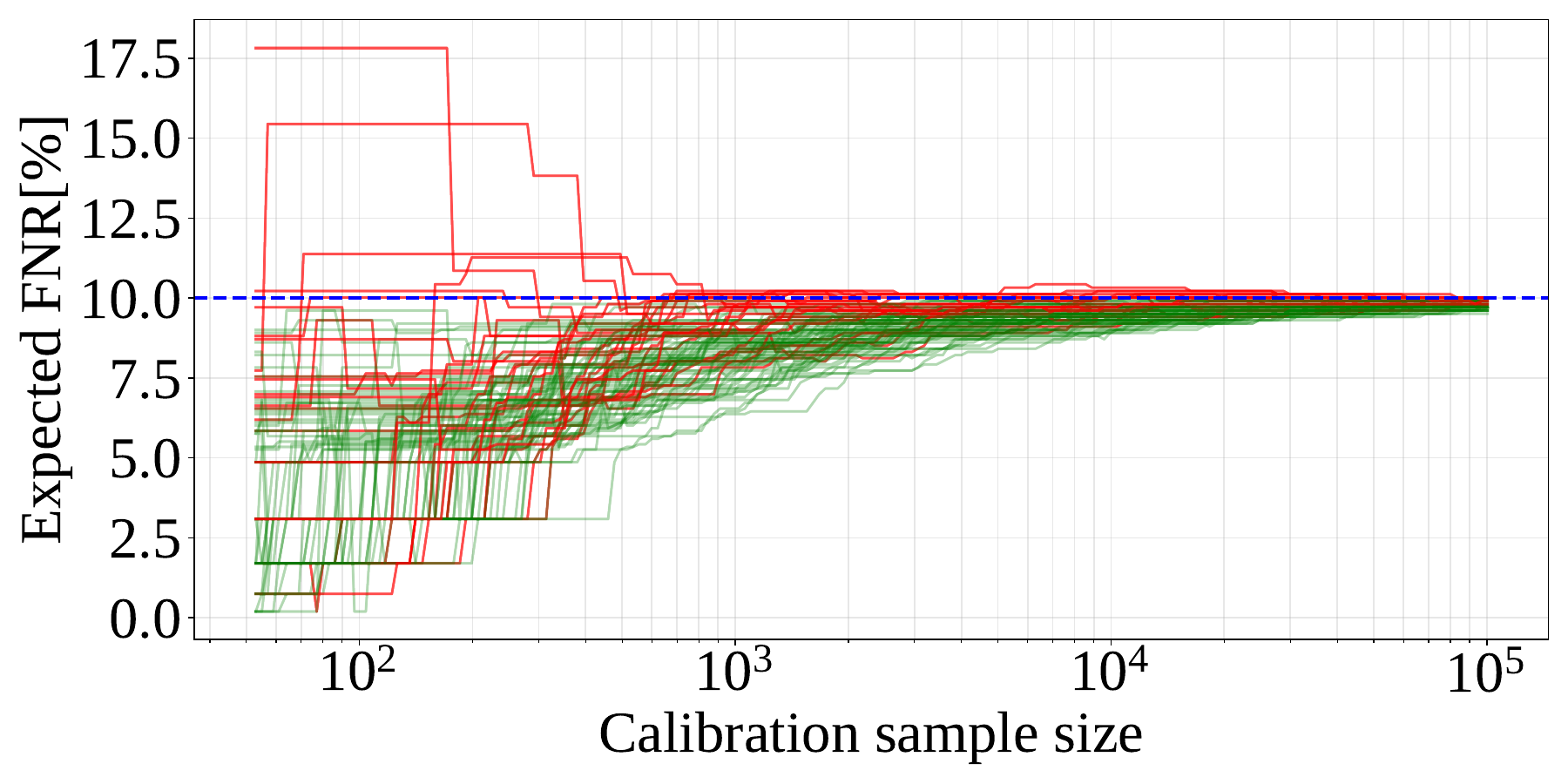}
        \caption{Fixed-time valid}
        \label{fig:fnrwsr}
    \end{subfigure}
    \hfill
    \begin{subfigure}[b]{0.45\textwidth}
        \centering
        \includegraphics[width=\textwidth]{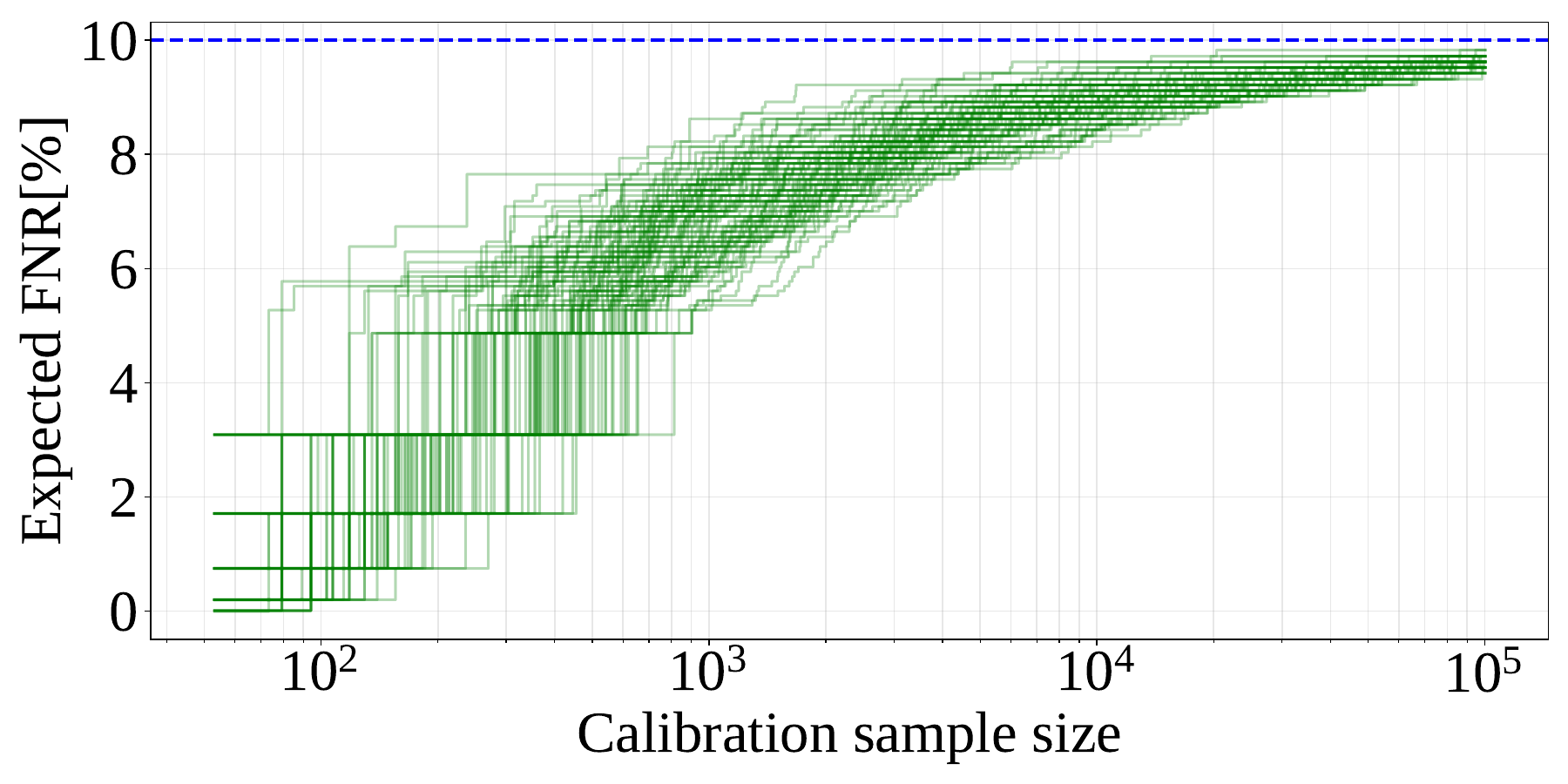}
        \caption{Anytime-valid}
        \label{fig:fnr}
    \end{subfigure}
    \caption{FNR of prediction sets 
    versus calibration sample size $n$, for the synthetic multilabel classification experiment with $\alpha = 10\%$, $\delta = 10\%$.}
    \label{fig:fnrsynt}
\end{figure*}

\textbf{False Negative Rate Control in Multilabel Classification.}
\label{sec:fnr}
 We use the multilabel classification setup of \citet{angelopoulos2022gentleintroductionconformalprediction} (Section~5.1). The data consist of images with multiple binary labels, and the prediction set is constructed as $\predictionset_\threshold(X) = \{ k : f_k(X) \ge 1-\threshold \}$, where $f_k(X) \in [0,1]$ is the estimated probability of label $k$. We target FNR control at level $\alpha = 10\%$ with $\delta = 10\%$.

Figure~\ref{fig:fnrcoco1} shows the realized FNR of the sequence of prediction sets constructed using Theorem~\ref{thm:main_riskcontrol}, compared to the fixed-time WSR bound of \citet{RCPS10.1145/3478535} and marginal risk control \citep{angelopoulos2024conformal}. Figure~\ref{fig:fnrcoco2} shows the corresponding average prediction set sizes.

\begin{figure*}[h!]
    \centering
    \begin{subfigure}[b]{0.45\textwidth}
        \centering
        \includegraphics[width=\textwidth]{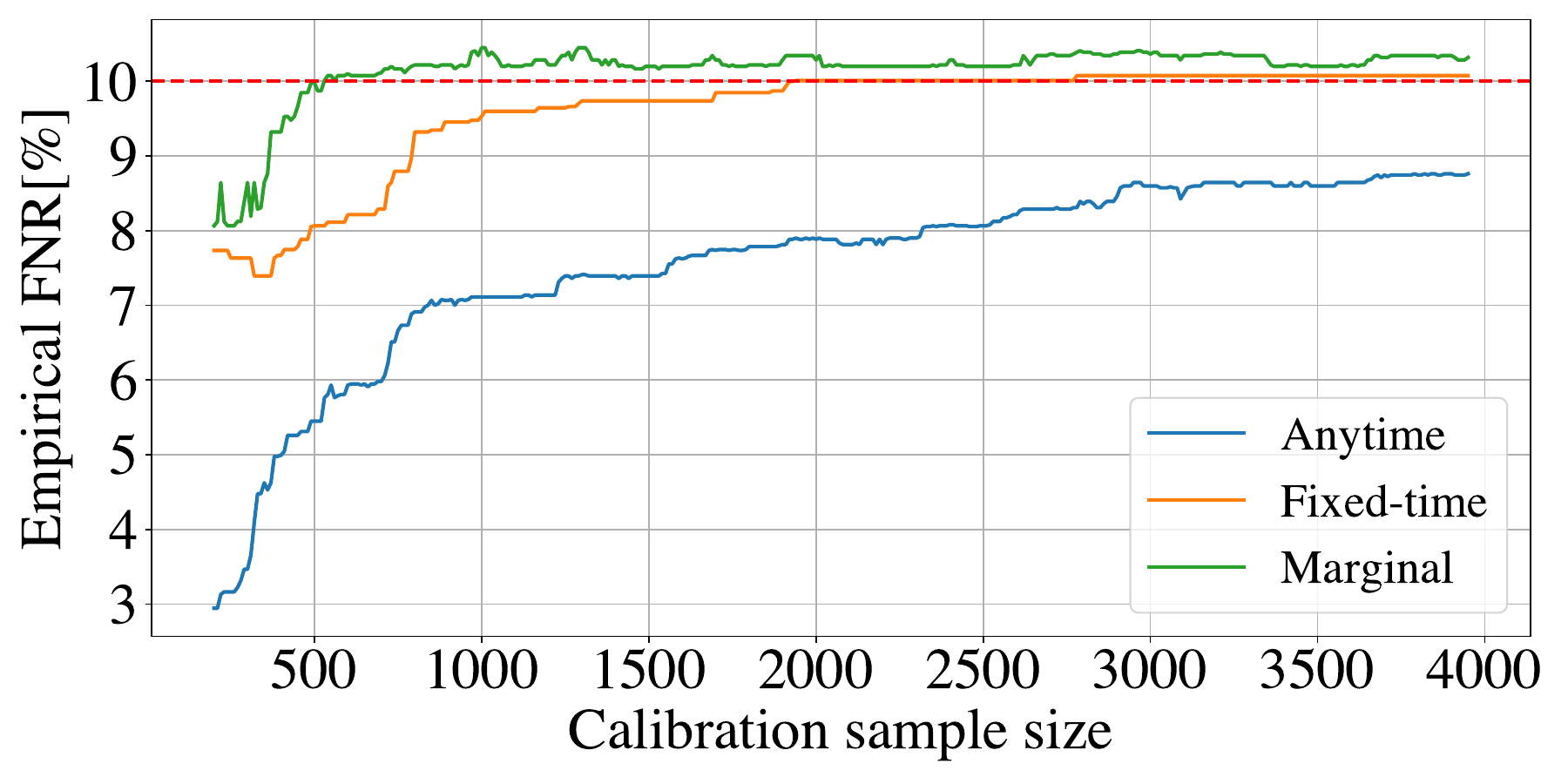}
        \caption{}
        \label{fig:fnrcoco1}
    \end{subfigure}
    \hfill
    \begin{subfigure}[b]{0.45\textwidth}
        \centering
        \includegraphics[width=\textwidth]{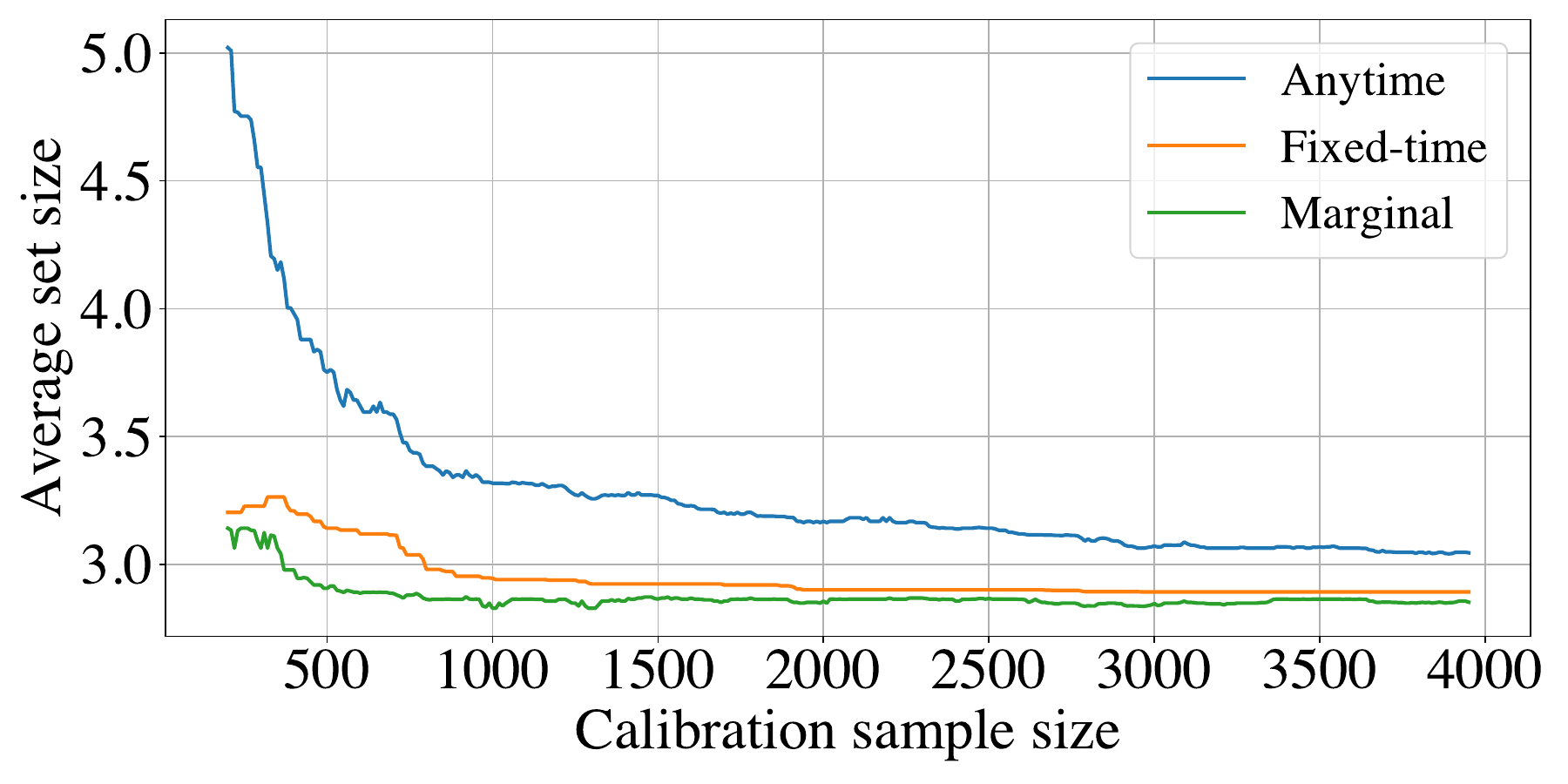}
        \caption{}
        \label{fig:fnrcoco2}
    \end{subfigure}
    \caption{Comparison of anytime-valid risk control using Theorem~\ref{thm:main_riskcontrol}, fixed-time valid risk control using the method of \citet{RCPS10.1145/3478535}, and marginal risk control for multilabel image classification with $\alpha = 10\%$, $\delta = 10\%$.}
    \label{fig:distshift_both}
\end{figure*}

\section{Conclusion}
We proposed an anytime-valid conformal risk control framework that guarantees, with high probability, that the risk of a sequence of prediction sets remains below a target level at all times while converging to it asymptotically. This is particularly useful in problems with cumulative data collection processes. The approach yields interpretable correction terms, tightness guarantees together with a proof that the rate of the correction is optimal, and extensions to test-time distribution shift. The framework applies to general bounded loss functions, including miscoverage, false negative rate, and others, as validated by our theoretical analysis and experiments.

A natural limitation is that the method is more conservative at small sample sizes, requiring additional data before producing informative prediction sets. Another limitation is that the distribution-shift setting requires knowledge of the importance weights $\omega(x,y)$; handling misspecified weights remains an open direction. Furthermore, future work includes extensions to online model updates as well as time-series models with uncertainty updates. This is an interesting direction as anytime-valid guarantees are especially relevant for models that need to run over extended periods in online settings. 

\begin{ack}
This work was partially supported by the Swedish Research Council (contract number 2024-03903). AHR is supported by eSSENCE, SciLifeLab, and the Wallenberg AI, Autonomous Systems and Software Program (WASP), funded by the Knut and Alice Wallenberg Foundation.
\end{ack}

\bibliographystyle{icml2026}

\newpage
\appendix
\onecolumn

\section{Background: Anytime-Valid Concentration Inequalities}
\label{app:stitch}

\textbf{A short introduction.} We briefly recall the main ideas behind time-uniform concentration; see \citet{Howard_2020,Howard2018TimeuniformNN} for a detailed treatment and \citet{overview} for an accessible overview. A process $(L_t)_{t\ge0}$ adapted to a filtration $(\mathcal{F}_t)_{t\ge0}$ is a \emph{supermartingale} if $\E[L_t\mid\mathcal{F}_{t-1}]\le L_{t-1}$ for all $t\ge1$, i.e., it does not increase in conditional expectation. It is commonly viewed as the wealth of a gambler in a game that is fair or unfavorable to them. The key tool is Ville's inequality. If $(L_t)$ is a nonnegative supermartingale with $\E[L_0]\le1$, then
\[
\mathbb{P}\big(\exists t\ge0:\ L_t\ge1/\delta\big)\le\delta.
\]
This is a time-uniform analogue of Markov's inequality, which only yields $\mathbb{P}(L_t\ge1/\delta)\le\delta$ for a single, fixed $t$.

For a sum $M_t=\sum_{i=1}^t\xi_i$ of independent mean-zero increments, a typical nonnegative supermartingale is of the form $L_t(\tau)=\exp(\tau M_t-\psi(\tau)V_t)$, where $\tau>0$ is a free parameter, $V_t$ is a variance process, and $\psi(\tau)V_t$ upper bounds the cumulant generating function of $\tau M_t$. Definition~\ref{th:canonical_assumption} below formalizes this condition. Ville's inequality then yields that, with probability at least $1-\delta$, $M_t<\log(1/\delta)/\tau+\psi(\tau)V_t/\tau$ for all $t$ simultaneously. This boundary is linear in $V_t$, and a fixed $\tau$ makes it tight only for $V_t$ near a single value. \emph{Stitching} \citep{Howard2018TimeuniformNN} overcomes this by splitting the range of $V_t$ into geometrically growing epochs $[\eta^km,\eta^{k+1}m)$, $k=0,1,\dots$, using a separate parameter $\tau_k$ tuned to epoch $k$ with error budget $\delta/h(k)$, and taking a union bound over the epochs. Since $\sum_k1/h(k)\le1$, the total error probability is at most $\delta$, and the resulting boundary grows as $\sqrt{V_t\log\log V_t}$, in line with the law of the iterated logarithm. Theorem~\ref{th:stitching} below states the result.

\begin{definition}[Sub-$\psi$ condition, Definition 1 \cite{Howard_2020}]
\label{th:canonical_assumption}
Let $\smash{(M_t)_{t = 0}^\infty, (V_t)_{t = 0}^\infty}$ be real-valued processes
adapted to an underlying filtration $(\mathcal{F}_t)_{t = 0}^\infty$ with $\smash{M_0=V_0=0}$
and~$\smash{V_t\geq0}$ for all $t$. For $c\in \R$, define 
 \begin{align*}
    \psi_{G,c}(\tau) = \frac{\tau^2}{2(1 - c\tau)}, \quad \text{for}\quad 0 \leq \tau <  \tau_{\max},
  \end{align*}
where $\tau_{\max}=1/\max\{0,c\}$ with the convention that $1/0=\infty$. We say that $(M_t)$ is a \emph{sub-gamma process} with scale parameter $c\in\R$ and variance process $(V_t)$ if, for each $\tau \in [0,\tau_{\max})$, there exists a supermartingale
$(L_t(\tau))_{t = 0}^\infty$ w.r.t. $(\mathcal{F}_t)$ such that
$\E[ L_0(\tau)] \leq 1$ and
  \begin{align*}
    \exp\left(\tau M_t - \psi_{G,c}(\tau) V_t \right) \leq L_t(\tau)
    \text{ a.s.\ for all } t.
  \end{align*}
\end{definition}

\citet{Howard2018TimeuniformNN} prove the following Theorem.
\begin{theorem}[Stitched boundary, Theorem 1 \cite{Howard2018TimeuniformNN}]\label{th:stitching}
    Let $(M_t)$ be a sub-gamma process with scale parameter $c\geq 0$ and variance process $(V_t)$. For any $\smash{\delta \in (0,1), \eta > 1, m>0}$, and
  $\smash{h : \R_{\geq 0} \to \R_{\geq 0}}$ increasing such that
  $\smash{\sum_{k=0}^\infty 1 / h(k) \leq 1}$, define
  \begin{equation}
\small
\mathcal{S}_\delta(v) =
  \sqrt{k_1^2 v \cdot d(v) + k_2^2 c^2 \cdot d^2(v)} \\
  + k_2 c \cdot d(v),
\text{ where }
\begin{cases}
d(v) = \log h(\log_{\eta}(\frac{v}{m})) + \log(\frac{l_0}{\delta}),\\
k_1 = (\eta^{1/4} + \eta^{-1/4}) / \sqrt{2},\\
k_2 = (\sqrt{\eta} + 1) / 2.
\end{cases}
\label{eq:stitching_operator}
\end{equation}
  
  Then,
  \begin{align*}
      \mathbb{P}(\forall t\geq1: M_t < \mathcal{S}_\delta(\max\{V_t,m\})) \geq 1-\delta.
  \end{align*}
  
\end{theorem}
As a direct corollary of the above Theorem, taking $\eta = 2$, $l_0 = 1$, and $h(k) = (k+1)^2\pi^2/6$. With $\eta=2$, Theorem~\ref{th:stitching} gives $k_1=2^{-1/4}+2^{-3/4}\approx1.4355$ and $k_2=(\sqrt2+1)/2\approx1.2071$, and $\sqrt{a+b}\le\sqrt a+\sqrt b$ yields $\mathcal{S}_\delta(v)\le k_1\sqrt{v\,d(v)}+2k_2\,c\,d(v)$. The constants $1.44$ and $2.42$ are $k_1$ and $2k_2$ rounded up for readability, we obtain the following explicit result.
\begin{corollary}
\label{cor:explicit_stitching_pi2}
Define
\[
d(v) = 2 \log\big(\log_2(v/m)+1\big) + \log\frac{\pi^2}{6 \delta}.
\]
 Let $(M_t)$ be a sub-gamma process with scale parameter $c\geq 0$ and variance process $(V_t)$. Then, for any $\delta \in (0,1)$, and $m>0$, 
\begin{align*}
      \mathbb{P}(\forall t\geq1: M_t < \tilde{\mathcal{S}}_\delta(\max\{V_t,m\})) \geq 1-\delta.
  \end{align*}
  
with 
\[
\tilde{\mathcal{S}}_\delta(v)
=
1.44\sqrt{ \, v \, d(v)} \;+\; 2.42 \, c \, d(v).
\]

\end{corollary}

\begin{remark}[Alternative constructions of correction terms] 
\label{app:fconstruction}
Our construction of the correction terms is based on
Corollary~\ref{cor:explicit_stitching_pi2}, which in turn relies on
Theorem~\ref{th:stitching} instantiated with specific parameter choices. Theorem~\ref{th:stitching} depends on the parameters
$\eta$ and $h$, which govern a trade-off between finite-sample constants and
asymptotic tightness. Smaller values of $\eta$ and more slowly growing
functions $h$ yield tighter asymptotic behavior, at the cost of larger
additive constants. Conversely, larger values lead to looser asymptotic
constants but improved behavior at smaller calibration sets. We refer the interested reader to
\cite{Howard2018TimeuniformNN} for further discussion.
\end{remark}

While Corollary~\ref{cor:explicit_stitching_pi2} provides a 
time-uniform bound on $(M_t)$, we will also need a fixed-time 
tail bound for a sub-gamma process. The following lemma provides 
this via a Markov inequality argument.

\begin{lemma}
\label{lem:help-chernoff}
    Let $c\geq0$. Let $(M_t)_{t\in \mathbb{N}}$ be sub-gamma with scale parameter $c$ and deterministic variance process $(V_t)_{t\in \mathbb{N}}$. Then, for all $n\in \mathbb{N}$, such that $V_n>0$, and $\delta \in (0,1)$,
    \[
\mathbb{P}\Bigg(
M_n \ge c \, \log \frac{1}{\delta} + \sqrt{2 V_n \, \log \frac{1}{\delta} + \left(c \, \log \frac{1}{\delta}\right)^2}
\Bigg) \le \delta.
\]
    
\end{lemma}
\begin{proof}
    For any $0\leq \tau<1/c$ and $x\in \R$, we have
    \begin{align*}
    \mathbb{P}(M_n\! \ge \!x)\!&= \!\mathbb{P}\big(\exp(\tau M_n\!-\!\psi_{G,c}(\tau) V_n)\!\ge\! e^{\tau x - V_n \psi_{G,c}(\tau )}\big)\\
    & \leq\exp\Big(-\tau x + V_n \psi_{G,c}(\tau)\Big),
\end{align*}
by Markov's inequality.
Set $\tau=\frac{x}{V_n + c x} \in [0,1/c)$ and set the RHS equal to $\delta$, that is 
\[
\exp\Big(-\frac{x^2}{2(V_n + c x)}\Big) = \delta.
\]
Then,
\[
x = c \, \log \frac{1}{\delta} + \sqrt{ 2 V_n \, \log \frac{1}{\delta} + \left(c \, \log \frac{1}{\delta}\right)^2 }.
\]
\end{proof}

\begin{lemma}[Proposition~S2 \cite{Howard2018TimeuniformNN}]
\label{lem:normalmixture}
Let $(M_t)_{t\in\mathbb{N}}$ be sub-gamma with scale parameter $c=0$ and possibly random variance process $(V_t)_{t\in\mathbb{N}}$. Then, for all $\rho>0$ and $\delta\in(0,1)$,
\[
\mathbb{P}\big(\exists t\ge1: M_t\ge u_\rho(V_t)\big)\le\delta,
\qquad
u_\rho(v)=\sqrt{2(v+\rho)\log\Big(\frac{1}{2\delta}\sqrt{\frac{v+\rho}{\rho}}+1\Big)}.
\]
In particular, $\mathbb{P}\big(M_n\ge u_\rho(V_n)\big)\le\delta$ for every fixed $n\in\mathbb{N}$.
\end{lemma}

\section{Derivations}
\label{sec:derivations}

For notational convenience, we denote the loss associated with the prediction set $\predictionset_{\threshold}(x)$ by 
\begin{equation*}
\loss_\threshold(x,y) :=\loss\left(\predictionset_{\threshold}(x),y\right).
\end{equation*}
The population risk is defined as $\Risk(\threshold)=\E[\loss_\threshold\sample]$.
To construct suitable sub-gamma processes, we will rely on certain properties of $\Risk$, and we first address the additional considerations that arise when $\Risk$ is discontinuous.
By assumption, $\loss_\threshold\sample$ is right-continuous in $\threshold$ and bounded. 
Hence, by the Dominated Convergence Theorem, $\Risk$ is right-continuous in $\threshold$. Furthermore, $\Risk$ is monotonically decreasing. For $\alpha \in [0,\sup_{\threshold\in \Threshold}\Risk(\threshold)]$, define
\[
\hat{\threshold}(\alpha) := \inf \{\threshold : \Risk(\threshold) \le \alpha\}.
\]
If $\alpha \geq \sup_{\lambda\in \Lambda}\Risk(\threshold)$, every threshold has risk at most $\alpha$, and the upper bounds hold trivially.
We also define the left-limit of $\Risk$ at $\hat{\threshold}(\alpha)$ by
\[
\Risk(\hat{\threshold}(\alpha)^-) := \lim_{\threshold \uparrow \hat{\threshold}(\alpha)} \Risk(\threshold),
\]
and, similarly, the corresponding left-limit of the loss function by
\[
\loss_{\hat{\threshold}(\alpha)^-}(x,y) := \lim_{\lambda \uparrow \hat{\threshold}(\alpha)} \loss_\lambda(x,y).
\]

Since $\loss_\threshold(x,y)$ is monotonically decreasing in $\threshold$ and bounded, it follows from the Dominated Convergence Theorem that
\begin{align*}
\Risk(\hat{\threshold}(\alpha)^-) 
&= \lim_{\lambda \uparrow \hat{\threshold}(\alpha)} \mathbb{E}[\loss_\lambda\sample]= \mathbb{E}\Big[\lim_{\lambda \uparrow \hat{\threshold}(\alpha)} \loss_\lambda\sample\Big] \\
&= \mathbb{E}[\loss_{\hat{\threshold}(\alpha)^-}\sample].
\end{align*}

By monotonicity and right-continuity of $\Risk$, we have
\[
\Risk(\hat{\threshold}(\alpha)) \le \alpha \le \Risk(\hat{\threshold}(\alpha)^-).
\]

The proofs require increments with mean exactly zero at the oracle threshold, i.e., $\E[\loss_{\hat\threshold(\alpha)}\sample]=\alpha$. When $\Risk$ is continuous at $\hat\threshold(\alpha)$, this holds automatically. However, $\Risk$ may jump across the level $\alpha$ at $\hat\threshold(\alpha)$, so that $\Risk(\hat\threshold(\alpha))<\alpha\le\Risk(\hat\threshold(\alpha)^-)$ and no threshold attains the level $\alpha$ exactly. The weight $\pi(\alpha)$ defined below is the interpolation coefficient that splits this jump. It defines an effective loss, a convex combination of the loss at $\hat\threshold(\alpha)$ and its left-limit, whose mean is exactly $\alpha$.

Define
\[
\pi(\alpha)
:= \begin{cases}
1, & \text{if } \Risk(\hat{\threshold}(\alpha)) = \alpha, \\[6pt]
\dfrac{\Risk(\hat{\threshold}(\alpha)^-) - \alpha}
      {\Risk(\hat{\threshold}(\alpha)^-) - \Risk(\hat{\threshold}(\alpha))},
& \text{if } \Risk(\hat{\threshold}(\alpha)) < \alpha.
\end{cases}
\]
Then $\pi(\alpha) \in [0,1]$, and
\[
\alpha
= \pi(\alpha)\, \Risk(\hat{\threshold}(\alpha))
+ (1-\pi(\alpha))\, \Risk(\hat{\threshold}(\alpha)^-).
\]

Equivalently,
\[
\mathbb{E}\Big[
\pi(\alpha)\, \loss_{\hat{\threshold}(\alpha)}\sample
+ (1-\pi(\alpha))\, \loss_{\hat{\threshold}(\alpha)^-}\sample
\Big]
= \alpha.
\]

To establish our main results, we construct suitable processes $(M_t)_{t = 0}^\infty$ and $(V_t)_{t = 0}^\infty$ that satisfy a sub-gamma condition, allowing us to apply Corollary~\ref{cor:explicit_stitching_pi2}.
\mainriskcontrol*
\begin{proof}[Proof of Theorem~\ref{thm:main_riskcontrol}]
For the growing calibration sequence $\{(X_i,Y_i)\}_{i \ge 1}$, define
\[
\xi_i(\alpha)=\pi(\alpha)\loss_{\hat{\threshold}(\alpha)}\sample[i]+(1-\pi(\alpha))\loss_{\hat{\threshold}(\alpha)^-}\sample[i],
\]
and let
\[
M_t(\alpha)=\sum_{i=1}^t (\alpha-\xi_i(\alpha)).
\]
By construction, $\E[\alpha-\xi_i(\alpha)]=0$ and $(\alpha-\xi_i(\alpha))\in[\alpha-B,\alpha]$. 
First, note that these are sufficient conditions for a sub-Bernoulli process \cite{Howard_2020}, (Table~3, row "Bernoulli~1"). Applying Proposition~2 of \cite{Howard_2020} (Table~5, rows~3 and~5) then yields that $(M_t(\alpha))_{t\in \mathbb{N}}$ is sub-gamma with scale parameter $c_\alpha =\max\left\{\frac{2\alpha-B}{3},0\right\}$, where we used monotonicity of the scale parameter, and variance process $(V_t)_{t\in \mathbb{N}}$ given by
\begin{equation}
    V_t=\alpha(B-\alpha)t.
\label{eq:processvariance}
\end{equation}

By Corollary~\ref{cor:explicit_stitching_pi2}, with $c=c_\alpha$ and $\tilde m=\alpha(B-\alpha)m^*$, we obtain 
\begin{align*}
    \mathbb{P}\left(\forall t\in \mathbb{N}: M_t(\alpha) < f_{c_\alpha,\tilde m,\delta}(\max\{\tilde m,V_t\})\right) \geq 1-\delta.
\end{align*}

Rewriting, yields, with probability at least $1-\delta$, for all $t$,

\begin{align*}
\alpha - \frac{f_{c_\alpha,\tilde m,\delta}(\max\{\tilde m,V_t\})}{t}
&<
\frac{1}{t}
\sum_{i=1}^t
\loss_{\hat{\threshold}(\alpha)^-}\sample[i]
.
\end{align*}

Thus, assuming the above is true, for $t\ge m^*$ we have $V_t\ge\tilde m$, so that $f_{c_\alpha,\tilde m,\delta}(\max\{\tilde m,V_t\})=t\correction_t$ with $\correction_t$ as in \eqref{eq:correctionterm_riskcontrol}, and
\begin{align*}
    \threshold_t&=\inf \left\{ \lambda : \Risk_t(\lambda) \le \alpha-\frac{f_{c_\alpha,\tilde m,\delta}(\max\{\tilde m,V_t\}) }{t}\right\}\\
     &\geq \inf \left\{ \lambda : \Risk_t(\lambda) < \frac{1}{t} \sum_{i=1}^t \loss_{\hat{\threshold}(\alpha)^-}\sample[i] \right\}\\ 
     &\geq \hat{\threshold}(\alpha),
\end{align*}
by monotonicity of the loss. For $t<m^*$, we have $\correction_t>\alpha$ by definition of $m^*$, so $\threshold_t=\thresholdmax\ge\hat\threshold(\alpha)$. Hence,
\begin{align*}
    \mathbb{P}\left(\forall t\in \mathbb{N}: \threshold_t\geq \hat{\threshold}(\alpha)\right) \geq 1-\delta.
\end{align*} 
\end{proof}

\lowerbound*

To prove Proposition~\ref{prop:lowerbound}, the key idea is to apply Lemma~\ref{lem:help-chernoff} to a sub-gamma process constructed at a fixed threshold $\lambda$, combined with a union bound over a dyadic grid.

\begin{proof}[Proof of Proposition~\ref{prop:lowerbound}]

By Theorem~\ref{thm:main_riskcontrol}, the upper bound holds for all $n$ with probability at least $1-\delta$. Taking a running minimum does not affect this, since it preserves $\threshold_n\ge\hat\threshold(\alpha)$. As $\E[\loss(\predictionset_{\threshold_n}(X),Y)\mid\calibration[n]]=\Risk(\threshold_n)$, it remains to show that, with probability at least $1-\delta$, $\Risk(\threshold_n)\ge\alpha-k_n$ for all $n$.

Fix $i\ge0$, let $N=2^i$, and set $\beta_i=\alpha-\correction_N-\correction'_N(\delta_{i+1})$. If $\beta_i\le0$, or if $\Risk(\threshold)>\beta_i$ for all $\threshold$, then $\Risk(\threshold_N)\ge\beta_i$ trivially. Otherwise, let $\check\threshold_i=\inf\{\threshold:\Risk(\threshold)\le\beta_i\}$, a deterministic threshold. By right-continuity, $\Risk(\check\threshold_i)\le\beta_i$, and since $\sup_\threshold\Risk(\threshold)\ge\alpha>\beta_i$, every $\threshold<\check\threshold_i$ satisfies $\Risk(\threshold)>\beta_i$, so that $\Risk(\check\threshold_i^-)\ge\beta_i$.

Let $\xi_j=\loss_{\check\threshold_i}(X_j,Y_j)-\Risk(\check\threshold_i)$. These are i.i.d., have mean zero, and take values in $[-\Risk(\check\threshold_i),B-\Risk(\check\threshold_i)]$. Hence, as in the proof of Theorem~\ref{thm:main_riskcontrol}, $\sum_{j\le t}\xi_j$ is sub-gamma with scale at most $B$ and variance process $\Risk(\check\threshold_i)(B-\Risk(\check\threshold_i))\,t\le\alpha(B-\alpha)t$. Lemma~\ref{lem:help-chernoff} at time $N$ thus shows that the event
\[
E_i=\big\{\Risk_N(\check\threshold_i)<\Risk(\check\threshold_i)+\correction'_N(\delta_{i+1})\big\}
\]
has probability at least $1-\delta_{i+1}$. On $E_i$, $\Risk_N(\check\threshold_i)<\beta_i+\correction'_N(\delta_{i+1})=\alpha-\correction_N$, so $\check\threshold_i$ is feasible in \eqref{eq:correctionterm} at sample size $N$, and hence $\threshold_N\le\check\threshold_i$, as the running minimum can only decrease $\threshold_N$. We distinguish two cases.

If $\threshold_N<\check\threshold_i$, then $\Risk(\threshold_N)\ge\Risk(\check\threshold_i^-)\ge\beta_i$.

If $\threshold_N=\check\threshold_i$, then every $\threshold<\threshold_N$ is infeasible at sample size $N$ (it lies below the infimum in \eqref{eq:correctionterm}), so $\Risk_N(\threshold)>\alpha-\correction_N$ for all $\threshold<\threshold_N$ and $\Risk_N(\threshold_N^-)\ge\alpha-\correction_N$. Hence, $\Risk_N(\check\threshold_i)=\Risk_N(\threshold_N^-)-J(\Risk_N,\threshold_N)\ge\alpha-\correction_N-d_N$, and, on $E_i$,
\[
\Risk(\threshold_N)=\Risk(\check\threshold_i)>\Risk_N(\check\threshold_i)-\correction'_N(\delta_{i+1})\ge\alpha-\correction_N-\correction'_N(\delta_{i+1})-d_N.
\]
In both cases, $\Risk(\threshold_N)\ge\alpha-k_N$.

Since $\sum_{i\ge0}\delta_{i+1}=\delta\sum_{i\ge1}\frac{1}{i(i+1)}=\delta$, a union bound shows that, with probability at least $1-\delta$, $\Risk(\threshold_{2^i})\ge\alpha-k_{2^i}$ for all $i\ge0$. For $2^i\le n<2^{i+1}$, we have $g(n)=2^i$, $k_n=k_{2^i}$, and $\threshold_n\le\threshold_{2^i}$, hence $\Risk(\threshold_n)\ge\Risk(\threshold_{2^i})\ge\alpha-k_n$. The rate follows because $\correction_{g(n)}$ and $\correction'_{g(n)}(\delta_{\lfloor\log_2n\rfloor+1})$ are of order $B\sqrt{n^{-1}\log\log n}$. Since only $J(\Risk_N,\threshold_N)\le d_N$ was used, $d_N$ may be replaced by $J(\Risk_N,\threshold_N)$.

\end{proof}

Before proving Theorem~\ref{thm:distshiftcontrol}, we note the following fact, which follows from a Taylor expansion argument.
\begin{lemma}\label{lemma:prepdistshift}
    Let $x\geq 0$, then
    \[
    \exp\left(x-\frac{x^2}{2}\right) \leq 1+x.
    \]
\end{lemma}

\distshiftcontrol*
\begin{proof}[Proof of Theorem~\ref{thm:distshiftcontrol}]
We proceed similarly as the proof of Theorem~\ref{thm:main_riskcontrol}, where $\Risk$, $\hat\threshold(\alpha)$, and $\pi(\alpha)$ now refer to the test-time risk $\Risk(\threshold)=\E_{P^*}[\loss_\threshold(X,Y)]$. For simplicity, we assume that $\Risk$ is continuous. If this is not the case, then as in the proof of Theorem~\ref{thm:main_riskcontrol}, we may replace each occurrence of $\loss_{\hat{\threshold}(\alpha)}$ with
\[
\left(\pi(\alpha)\loss_{\hat{\threshold}(\alpha)}\sample[i]+(1-\pi(\alpha))\loss_{\hat{\threshold}(\alpha)^-}\sample[i]\right).
\]
Define 
\[
\xi_i(\alpha)=\omega\sample[i]\left(B-\loss_{\hat{\threshold}(\alpha)}\sample[i]\right)
\]
and 
\[
\mu(\alpha)=\E\left[\xi_i(\alpha)\right]=B-\alpha.
\]
Set
\[
M_t(\alpha)=\sum_{i=1}^t \left(\xi_i(\alpha)-\mu(\alpha)\right).
\]
For any $\tau \geq 0$, note that 
\begin{align*}
    &\E\left[\exp\left(\tau(\xi_i(\alpha)-\mu(\alpha))-\frac{\tau^2\xi_i(\alpha)^2}{2}\right)\right]\\
    &= \exp(-\tau \mu(\alpha))\E\left[\exp\left(\tau\xi_i(\alpha)-\frac{\tau^2\xi_i(\alpha)^2}{2}\right)\right]\\
    &\leq \exp(-\tau \mu(\alpha))\E \left[1+ \tau\xi_i(\alpha)\right]\\
    &= \exp(-\tau \mu(\alpha))(1+ \tau\mu(\alpha))\\
    &\leq \exp(-\tau \mu(\alpha))\exp(\tau \mu(\alpha))=1
\end{align*}
due to Lemma~\ref{lemma:prepdistshift}, since $\xi_i(\alpha)\tau \geq 0$, and the fact that $1+x\leq \exp(x)$, for all $x\in \R$.

Hence,  
$(M_t(\alpha))_{t\in \mathbb{N}}$ is sub-gamma with scale parameter $0$ and variance process $(V_t)_{t\in \mathbb{N}}$  given by
\[
V_t=\sum_{i=1}^t \omega_i^2 \left(B-\loss_{\hat{\threshold}(\alpha)}\sample[i] \right)^2 \leq B^2W_t.
\]

By Corollary~\ref{cor:explicit_stitching_pi2}, with $c=0$ and the fixed $m$, we obtain,
\begin{align*}
    \mathbb{P}\left(\forall t\in \mathbb{N}: M_t(\alpha) < b_{m,\delta}(\max\{m,B^2W_t\})\right) \geq 1-\delta.
\end{align*}
Rewriting, yields
\begin{align*}
     \mathbb{P}\left(\forall t\in \mathbb{N}: \alpha-\correction_t < \Risk_t^\omega(\hat{\threshold}(\alpha))\right) \geq 1-\delta,
\end{align*}
for
\[
 \correction_t= \frac{b_{m,\delta}(\max\{m,B^2W_t\}) }{t}+B\left(1-\frac{1}{t}\sum_{i=1}^t\omega_i\right).
\]

The result follows by monotonicity and right-continuity of $\Risk$, exactly as in the proof of Theorem~\ref{thm:main_riskcontrol}.
\end{proof}

\begin{remark}\label{app:c=0}
In the proofs of Theorems~\ref{thm:main_riskcontrol} and~\ref{thm:distshiftcontrol}, we apply Corollary~\ref{cor:explicit_stitching_pi2} to processes of different nature. In the i.i.d.\ setting, the increments $\alpha-\xi_i(\alpha)$ have mean zero and lie in $[\alpha-B,\alpha]$. Only their upper tail matters, and it is bounded by $\alpha$. For $\alpha\le B/2$, this asymmetry makes the process sub-Gaussian with variance process $\alpha(B-\alpha)t$, so no range term is needed ($c_\alpha=0$). Under distribution shift, the increments $\xi_i(\alpha)=\omega_i(B-\loss_{\hat\lambda(\alpha)}(X_i,Y_i))$ are nonnegative but unbounded, since they involve the importance weights. Without assumptions on $\omega_i$, we control them through their observed squares (self-normalization), which yields the variance process $\sum_{i\le t}\xi_i(\alpha)^2\le B^2W_t$.

\end{remark}

\distshiftcontrollower*
\begin{proof}[Proof of Proposition~\ref{prop:distshiftlower}]
Write $\Risk(\threshold)=\E_{P^*}[\loss_\threshold(X,Y)]$ for the test-time risk. By Theorem~\ref{thm:distshiftcontrol}, the upper bound holds for all $n$ with probability at least $1-\delta$. Since $\E_{P^*}[\loss(\predictionset_{\threshold_n}(X),Y)\mid\calibration[n]]=\Risk(\threshold_n)$, it remains to show that, with probability at least $1-\delta$, $\Risk(\threshold_n)\ge\alpha-k_n$ for all $n$. We follow the proof of Proposition~\ref{prop:lowerbound}, with two differences. The weighted losses may be unbounded, so their deviations are controlled by Corollary~\ref{cor:explicit_stitching_pi2}, which allows for a random variance process and $\correction_N$ is random, so we use a grid of deterministic thresholds and a union bound over the grid.

Fix $i\ge0$ and let $N=2^i$, so that $\lfloor\log_2N\rfloor+1=i+1$ and $\delta_{N,j}=\delta_{i+1}/((j+1)(j+2))$. For $j\ge0$, let $\beta_j=\alpha-2^j/N$. Call $j$ active if $\beta_j>0$ and $\Risk(\threshold)\le\beta_j$ for some $\threshold$. For active $j$, let $\check\threshold_j=\inf\{\threshold:\Risk(\threshold)\le\beta_j\}$, a deterministic threshold. By right-continuity, $\Risk(\check\threshold_j)\le\beta_j$, and every $\threshold<\check\threshold_j$ satisfies $\Risk(\threshold)>\beta_j$.

For active $j$, let $Z_l=\omega_l\loss_{\check\threshold_j}(X_l,Y_l)\ge0$, $l\ge1$. These are i.i.d.\ with $\E_P[Z_l]=\Risk(\check\threshold_j)$. Applying Lemma~\ref{lemma:prepdistshift} with $x=\tau Z_l\ge0$ shows, exactly as in the proof of Theorem~\ref{thm:distshiftcontrol}, that $\sum_{l\le t}(Z_l-\Risk(\check\threshold_j))$ is sub-gamma with scale $0$ and variance process $\sum_{l\le t}Z_l^2\le B^2W_t$. By Corollary~\ref{cor:explicit_stitching_pi2} with parameter $m$ and error probability $\delta_{N,j}$, evaluated at time $N$, and since $b_{m,\delta}(\max\{\cdot,m\})$ is non-decreasing, the event
\[
E_{i,j}=\big\{\Risk^\omega_N(\check\threshold_j)<\Risk(\check\threshold_j)+u_{N,j}\big\},
\]
where $u_{N,j}=b_{m,\delta_{N,j}}(\max\{B^2W_N,m\})/N$, has probability at least $1-\delta_{N,j}$. Since $\sum_{j\ge0}\frac{1}{(j+1)(j+2)}=1$ and $\sum_{i\ge0}\delta_{i+1}=\delta$, all events $E_{i,j}$, $i\ge0$, $j$ active, hold simultaneously with probability at least $1-\delta$. We work on this event.

Let $j^*=j_N$. By the definition of $\correction'_N$, we have $\correction_N+\correction'_N=2^{j^*}/N$, so that
\[
\beta_{j^*}=\alpha-2^{j^*}/N=\alpha-\correction_N-\correction'_N=\alpha-k_N+d_N,
\]
and, by the definition of $j_N$, $\correction'_N=2^{j^*}/N-\correction_N\ge u_{N,j^*}$. If $j^*$ is not active, then $\Risk(\threshold_N)\ge\beta_{j^*}\ge\alpha-k_N$ trivially. Otherwise, on $E_{i,j^*}$,
\[
\Risk^\omega_N(\check\threshold_{j^*})< \Risk(\check\threshold_{j^*})+u_{N,j^*}\leq\beta_{j^*}+u_{N,j^*}\le\beta_{j^*}+\correction'_N=\alpha-\correction_N,
\]
so $\check\threshold_{j^*}$ is feasible in \eqref{eq:correctionterm} at sample size $N$, and $\threshold_N\le\check\threshold_{j^*}$, as the running minimum can only decrease $\threshold_N$. We distinguish two cases.

If $\threshold_N<\check\threshold_{j^*}$, then $\Risk(\threshold_N)>\beta_{j^*}\ge\alpha-k_N$ by definition of $\check\threshold_{j^*}$.

If $\threshold_N=\check\threshold_{j^*}$, then every $\threshold<\threshold_N$ is infeasible at sample size $N$, so $\Risk^\omega_N(\threshold_N^-)\ge\alpha-\correction_N$ and $\Risk^\omega_N(\check\threshold_{j^*})=\Risk^\omega_N(\threshold_N^-)-J(\Risk^\omega_N,\threshold_N)\ge\alpha-\correction_N-d_N$. On $E_{i,j^*}$, therefore,
\[
\Risk(\threshold_N)=\Risk(\check\threshold_{j^*})>\Risk^\omega_N(\check\threshold_{j^*})-u_{N,j^*}\ge\alpha-\correction_N-u_{N,j^*}-d_N\ge\alpha-\correction_N-\correction'_N-d_N=\alpha-k_N.
\]
In both cases, $\Risk(\threshold_N)\ge\alpha-k_N$.

For $2^i\le n<2^{i+1}$, we have $g(n)=2^i$, $k_n=k_{2^i}$, and $\threshold_n\le\threshold_{2^i}$, hence $\Risk(\threshold_n)\ge\Risk(\threshold_{2^i})\ge\alpha-k_n$. A union bound with the event of Theorem~\ref{thm:distshiftcontrol} gives probability at least $1-2\delta$.

\emph{Asymptotic tightness.} Since $d_n\to0$ by assumption, it suffices to show $\correction_n\to0$ and $\correction'_n\to0$. Write $v_n=\max\{B^2W_n,m\}$. First, $\correction_n=B(1-\frac1n\sum_{l\le n}\omega_l)+b_{m,\delta}(v_n)/n\to0$, where the first term vanishes because $\frac1n\sum_{l\le n}\omega_l\to1$, and the second decays as $O(\sqrt{v_n\log\log v_n}/n)\to0$.

Second, the grid is geometric, so rounding up to the next grid point at most doubles the distance. If $j_n\ge1$, minimality of $j_n$ gives $2^{j_n-1}/n<\correction_n+u_{n,j_n-1}\le\correction_n+u_{n,j_n}$, since $u_{n,j}$ is non-decreasing in $j$ ($\delta_{n,j}$ decreases in $j$), and hence $\correction'_n=2^{j_n}/n-\correction_n<\correction_n+2u_{n,j_n}$. If $j_n=0$, then $\correction'_n=1/n-\correction_n$. Together with the definition of $j_n$, this gives $u_{n,j_n}\le\correction'_n\le|\correction_n|+2u_{n,j_n}+1/n$.

It remains to show $u_{n,j_n}\to0$. For $j\le\log_2n+2$ and $i=\lfloor\log_2n\rfloor+1$, we have $\log(1/\delta_{n,j})=\log(1/\delta)+\log(i(i+1))+\log((j+1)(j+2))=O(\log\log n)$, so that $u_{n,j}=O\big(\sqrt{v_n(\log\log v_n+\log\log n)}/n\big)$ uniformly in such $j$. This tends to zero as $v_n\log\log v_n/n^2\to0$ by assumption, and $v_n\log\log n/n^2\to0$ because it is at most of order $v_n\log\log v_n/n^2$ if $v_n\ge n/2$, and at most $\log\log n/(2n)$ otherwise. In particular, $\bar j=\lceil\log_2n\rceil+1$ satisfies $2^{\bar j}/n\ge2\ge\correction_n+u_{n,\bar j}$ for large $n$, so that $j_n\le\bar j$, and hence $u_{n,j_n}\to0$. Therefore, $\correction'_n\to0$ and $k_n\to0$.
\end{proof}

\begin{proof}[Proof of Proposition~\ref{prop:rateoptimal}]
Since the miscoverage loss is $\{0,1\}$-valued, nonincreasing, and right-continuous in $\threshold$, we can write $\loss_\threshold(X_i,Y_i)=\mathbf{1}\{\threshold<T_i\}$ with $T_i=\inf\{\threshold\in\Threshold:Y_i\in\predictionset_\threshold(X_i)\}$. Note that for $\predictionset_\threshold(x)=\{y:s(x,y)\le\threshold\}$, $T_i=s(X_i,Y_i)$ is simply the score. Then $\Risk(\threshold)=\mathbb{P}(T_1>\threshold)$ and $\Risk_n(\threshold)=\frac1n\sum_{i=1}^n\mathbf{1}\{T_i>\threshold\}$. Write $\hat\threshold=\hat\threshold(\alpha)$.

By definition of $\hat\threshold$, $\Risk(\threshold)>\alpha$ for all $\threshold<\hat\threshold$, and $\Risk(\hat\threshold)\le\alpha$ by right-continuity. Since $\hat\threshold$ is an interior point of $\Threshold$ and $\Risk$ is continuous at $\hat\threshold$, it follows that $\Risk(\hat\threshold)=\alpha$ and $\mathbb{P}(T_1=\hat\threshold)=\Risk(\hat\threshold^-)-\Risk(\hat\threshold)=0$. Hence, almost surely, no $T_i$ equals $\hat\threshold$.

Let $\xi_i=\loss_{\hat\threshold}(X_i,Y_i)-\alpha$. The $\xi_i$ are i.i.d.\ with mean zero and variance $\sigma^2=\alpha(1-\alpha)>0$. By the law of the iterated logarithm (see, e.g., Theorem~8.5.2 in \citealp{Durrett2019}),
\[
\liminf_{n\to\infty}\frac{\sum_{i=1}^n\xi_i}{\sqrt{2\sigma^2n\log\log n}}=-1\qquad\text{a.s.}
\]
Fix $\epsilon\in(0,1)$. Almost surely, $\sum_{i=1}^n\xi_i<-(1-\epsilon)\sqrt{2\sigma^2n\log\log n}$ for infinitely many $n$. Since $\correction_n=o(\sqrt{n^{-1}\log\log n})$, we have $n\correction_n<(1-\epsilon)\sqrt{2\sigma^2n\log\log n}$ for all sufficiently large $n$. Consequently, almost surely,
\[
\Risk_n(\hat\threshold)=\alpha+\frac1n\sum_{i=1}^n\xi_i<\alpha-\correction_n\qquad\text{for infinitely many }n.
\]

Fix such an $n$ and assume that no $T_i$ equals $\hat\threshold$. As there are only finitely many $T_i$ and $\hat\threshold$ is an interior point of the interval $\Threshold$, there exists $\threshold^\circ\in\Threshold$ with $\threshold^\circ<\hat\threshold$ such that no $T_i$ lies in $[\threshold^\circ,\hat\threshold]$. Then $\Risk_n(\threshold^\circ)=\Risk_n(\hat\threshold)<\alpha-\correction_n$, so $\threshold^\circ$ is feasible in \eqref{eq:correctionterm}, and therefore $\threshold_n\le\threshold^\circ<\hat\threshold$. By the first step, $\Risk(\threshold_n)>\alpha$. Since $\threshold_n$ is a function of $\calibration[n]$ and the test point is independent of $\calibration[n]$, we obtain $\E[\loss(\predictionset_{\threshold_n}(X),Y)\mid\calibration[n]]=\Risk(\threshold_n)>\alpha$.

Hence, almost surely, the risk exceeds $\alpha$ for infinitely many $n$.
\end{proof}

\section{Fixed-Time Sample-Conditional Bounds under Covariate Shift}
\label{app:fixedtime}
We derive a fixed-time analogue of Theorem~\ref{thm:distshiftcontrol}, 
obtaining a sharper correction term at the cost of uniformity in $n$.
Recall that $W_n = \sum_{i=1}^n \omega_i^2$ and that the importance-weighted 
empirical risk is
\[
\Risk_n^\omega(\lambda) 
= \frac{1}{n}\sum_{i=1}^n \omega_i\,\loss_\lambda(X_i,Y_i).
\]

\begin{theorem}[Fixed-time sample-conditional risk control under covariate shift]
\label{thm:fixedtime_shift}
Let $\loss$ be an arbitrary loss function bounded in $[0,B]$, monotone 
in $\lambda$, and right-continuous in $\lambda$. 
For fixed $n \in \mathbb{N}$ and $\delta \in (0,1)$, let $\rho>0$ be chosen independently of the calibration data (it may depend on $n$), and construct $\lambda_n$ 
via~\eqref{eq:correctionterm}, with $\Risk_n$ replaced by $\Risk_n^\omega$, with correction term
\begin{equation}
    \correction_n 
    = B\!\left(1 - \frac{1}{n}\sum_{i=1}^n \omega_i\right) 
    + \frac{u_\rho(B^2W_n)}{n},\qquad u_\rho(v)=\sqrt{2(v+\rho)\log\Big(\frac{1}{2\delta}\sqrt{\frac{v+\rho}{\rho}}+1\Big)}.
\label{eq:correction_fixedtime_shift}
\end{equation}
Then
\[
\mathbb{P}\!\left(
    \mathbb{E}_{P^*}\!\bigl[
        \loss(\mathcal{C}_{\lambda_n}(X),Y) \mid \calibration[n]
    \bigr] \leq \alpha
\right) \geq 1 - \delta.
\]
\end{theorem}
\begin{remark}
If $\rho$ does not depend on $n$, the guarantee holds simultaneously for all $n$ (Lemma~\ref{lem:normalmixture}). Choosing $\rho=B^2n/(-W_{-1}(-4\delta^2/e)-1)$, where $W_{-1}$ is the lower branch of the Lambert $W$ function, approximately minimizes the boundary at $B^2W_n\approx B^2n$ \citep[Proposition~3]{Howard2018TimeuniformNN}. For $\delta=10\%$, this gives $\rho\approx B^2n/5$.
\end{remark}

\begin{proof}[Proof of Theorem~\ref{thm:fixedtime_shift}]
The proof is identical to that of Theorem~\ref{thm:distshiftcontrol}, with
Corollary~\ref{cor:explicit_stitching_pi2} replaced by Lemma~\ref{lem:normalmixture}.
With $M_t(\alpha)$ and $V_t\le B^2W_t$ as defined there, Lemma~\ref{lem:normalmixture}
and the fact that $u_\rho$ is non-decreasing give
\[
\mathbb{P}\big(\forall t\in\mathbb{N}: M_t(\alpha)<u_\rho(B^2W_t)\big)\ge 1-\delta .
\]
The same rearrangement yields
$\mathbb{P}\big(\forall t\in\mathbb{N}: \alpha-\correction_t<\Risk_t^\omega(\hat\threshold(\alpha))\big)\ge1-\delta$
with
$\correction_t=u_\rho(B^2W_t)/t+B\big(1-\frac1t\sum_{i=1}^t\omega_i\big)$, in particular at
$t=n$. The result follows as in the proof of Theorem~\ref{thm:distshiftcontrol}.
\end{proof}

\begin{proposition}[Lower bound, fixed-time under covariate shift]
\label{prop:fixedtime_lower}
In the setting of Theorem~\ref{thm:fixedtime_shift}, assume that $\alpha\le\sup_{\threshold\in\Threshold}\Risk(\threshold)$, where $\Risk(\threshold)=\E_{P^*}[\loss_\threshold(X,Y)]$, and let $d_n=\sup_{\threshold\in\Threshold}J(\Risk^\omega_n,\threshold)$. Let $u_\rho(v;\delta')$ denote $u_\rho(v)$ from \eqref{eq:correction_fixedtime_shift} with $\delta$ replaced by $\delta'$, and, for $j\ge0$, let
\[
\delta_j=\frac{\delta}{(j+1)(j+2)},\qquad u_j=\frac{u_\rho(B^2W_n;\delta_j)}{n},\qquad j_n=\min\Big\{j\ge0:\ \frac{2^j}{\sqrt n}\ge\correction_n+u_j\Big\}.
\]
Then, with probability at least $1-2\delta$,
\[
\alpha-k_n\le\E_{P^*}\big[\loss(\predictionset_{\threshold_n}(X),Y)\mid\calibration[n]\big]\le\alpha,\qquad k_n=d_n+\correction_n+\correction'_n,\qquad \correction'_n=\frac{2^{j_n}}{\sqrt n}-\correction_n,
\]
where $\correction'_n$ is computable from the calibration data.
\end{proposition}

\begin{proof}
The upper bound holds with probability at least $1-\delta$ by Theorem~\ref{thm:fixedtime_shift}. For the lower bound, we follow the proof of Proposition~\ref{prop:distshiftlower} at the fixed sample size $n$, with the grid $\beta_j=\alpha-2^j/\sqrt n$, $j\ge0$, and with Lemma~\ref{lem:normalmixture} in place of Corollary~\ref{cor:explicit_stitching_pi2}. Call $j$ active if $\beta_j>0$ and $\Risk(\threshold)\le\beta_j$ for some $\threshold$, and, for active $j$, let $\check\threshold_j=\inf\{\threshold:\Risk(\threshold)\le\beta_j\}$. Then $\Risk(\check\threshold_j)\le\beta_j$, and every $\threshold<\check\threshold_j$ satisfies $\Risk(\threshold)>\beta_j$.

For active $j$, the variables $Z_l=\omega_l\loss_{\check\threshold_j}(X_l,Y_l)\ge0$ are i.i.d.\ with mean $\Risk(\check\threshold_j)$, and, as in the proof of Theorem~\ref{thm:fixedtime_shift}, $\big(\exp(\tau\sum_{l\le t}(Z_l-\Risk(\check\threshold_j))-\tau^2\sum_{l\le t}Z_l^2/2)\big)_{t\ge0}$ is a supermartingale for every $\tau\ge0$. Since $\sum_{l\le n}Z_l^2\le B^2W_n$ and $u_\rho(\cdot\,;\delta_j)$ is non-decreasing, Lemma~\ref{lem:normalmixture} with error probability $\delta_j$ shows that the event
\[
E_j=\big\{\Risk^\omega_n(\check\threshold_j)<\Risk(\check\threshold_j)+u_j\big\}
\]
has probability at least $1-\delta_j$. As $\sum_{j\ge0}\delta_j=\delta$, all events $E_j$, $j$ active, hold simultaneously with probability at least $1-\delta$. We work on this event.

Let $j^*=j_n$. By the definitions of $\correction'_n$ and $j_n$, we have $\beta_{j^*}=\alpha-\correction_n-\correction'_n=\alpha-k_n+d_n$ and $\correction'_n\ge u_{j^*}$. If $j^*$ is not active, then $\Risk(\threshold_n)\ge\beta_{j^*}\ge\alpha-k_n$ trivially. Otherwise, on $E_{j^*}$, $\Risk^\omega_n(\check\threshold_{j^*})<\beta_{j^*}+u_{j^*}\le\beta_{j^*}+\correction'_n=\alpha-\correction_n$, so $\check\threshold_{j^*}$ is feasible in \eqref{eq:correctionterm} and $\threshold_n\le\check\threshold_{j^*}$. If $\threshold_n<\check\threshold_{j^*}$, then $\Risk(\threshold_n)>\beta_{j^*}\ge\alpha-k_n$ by definition of $\check\threshold_{j^*}$. If $\threshold_n=\check\threshold_{j^*}$, then every $\threshold<\threshold_n$ is infeasible, so $\Risk^\omega_n(\threshold_n^-)\ge\alpha-\correction_n$ and $\Risk^\omega_n(\check\threshold_{j^*})\ge\alpha-\correction_n-d_n$; on $E_{j^*}$, therefore,
\[
\Risk(\threshold_n)=\Risk(\check\threshold_{j^*})>\Risk^\omega_n(\check\threshold_{j^*})-u_{j^*}\ge\alpha-\correction_n-\correction'_n-d_n=\alpha-k_n.
\]
A union bound with the event of Theorem~\ref{thm:fixedtime_shift} completes the proof.
\end{proof}

\newpage
\section*{NeurIPS Paper Checklist}

\begin{enumerate}

\item {\bf Claims}
    \item[] Question: Do the main claims made in the abstract and introduction accurately reflect the paper's contributions and scope?
    \item[] Answer: \answerYes{} 
    \item[] Justification: The abstract and introduction clearly state the three main contributions: anytime-valid risk control, a matching lower bound proving asymptotic optimality, and extension to distribution shift. These align well with Theorems 4.1, 4.8, and Propositions 4.6, 4.11 and 4.12.

\item {\bf Limitations}
    \item[] Question: Does the paper discuss the limitations of the work performed by the authors?
    \item[] Answer: \answerYes{} 
    \item[] Justification: The limitations are discussed in Section 6, where we also outline directions for improvement.

\item {\bf Theory assumptions and proofs}
    \item[] Question: For each theoretical result, does the paper provide the full set of assumptions and a complete (and correct) proof?
    \item[] Answer: \answerYes{} 
    \item[] Justification: All theorems state their assumptions explicitly, and the proofs are provided in the appendix.

    \item {\bf Experimental result reproducibility}
    \item[] Question: Does the paper fully disclose all the information needed to reproduce the main experimental results of the paper to the extent that it affects the main claims and/or conclusions of the paper (regardless of whether the code and data are provided or not)?
    \item[] Answer: \answerYes{} 
    \item[] Justification: All experiments are fully specified in the paper. The synthetic experiments include complete data generating processes, model choices, score functions, and parameter settings ($\alpha$, $\delta$, sample sizes). The real-data experiments reference established baseline implementations. Code to reproduce all experimental results is provided in the Supplementary Material.

\item {\bf Open access to data and code}
    \item[] Question: Does the paper provide open access to the data and code, with sufficient instructions to faithfully reproduce the main experimental results, as described in supplemental material?
    \item[] Answer: \answerYes{} 
    \item[] Justification: Code to reproduce the experimental results is provided in the Supplementary Material.

\item {\bf Experimental setting/details}
    \item[] Question: Does the paper specify all the training and test details (e.g., data splits, hyperparameters, how they were chosen, type of optimizer) necessary to understand the results?
    \item[] Answer: \answerYes{} 
    \item[] Justification: All training and test details are fully specified in the paper and supplementary material.

\item {\bf Experiment statistical significance}
    \item[] Question: Does the paper report error bars suitably and correctly defined or other appropriate information about the statistical significance of the experiments?
    \item[] Answer: \answerYes{} 
    \item[] Justification: The figures show distributions of miscoverage rates over many independent runs, directly illustrating variability across random calibration datasets.

\item {\bf Experiments compute resources}
    \item[] Question: For each experiment, does the paper provide sufficient information on the computer resources (type of compute workers, memory, time of execution) needed to reproduce the experiments?
    \item[] Answer: \answerNo{} 
    \item[] Justification: The experiments are lightweight and do not require significant compute resources.
    
\item {\bf Code of ethics}
    \item[] Question: Does the research conducted in the paper conform, in every respect, with the NeurIPS Code of Ethics \url{https://neurips.cc/public/EthicsGuidelines}?
    \item[] Answer: \answerYes{} 
    \item[] Justification: The research conforms with the NeurIPS Code of Ethics.

\item {\bf Broader impacts}
    \item[] Question: Does the paper discuss both potential positive societal impacts and negative societal impacts of the work performed?
    \item[] Answer: \answerNA{} 
    \item[] Justification: This is a methods paper, and no immediate societal impacts are apparent.
    
\item {\bf Safeguards}
    \item[] Question: Does the paper describe safeguards that have been put in place for responsible release of data or models that have a high risk for misuse (e.g., pre-trained language models, image generators, or scraped datasets)?
    \item[] Answer: \answerNA{} 
    \item[] Justification: No models, datasets, or high-risk assets are released.

\item {\bf Licenses for existing assets}
    \item[] Question: Are the creators or original owners of assets (e.g., code, data, models), used in the paper, properly credited and are the license and terms of use explicitly mentioned and properly respected?
    \item[] Answer: \answerYes{} 
    \item[] Justification: All datasets and code used are properly cited with references to their original sources.

\item {\bf New assets}
    \item[] Question: Are new assets introduced in the paper well documented and is the documentation provided alongside the assets?
    \item[] Answer: \answerNA{} 
    \item[] Justification: We do not release new assets.

\item {\bf Crowdsourcing and research with human subjects}
    \item[] Question: For crowdsourcing experiments and research with human subjects, does the paper include the full text of instructions given to participants and screenshots, if applicable, as well as details about compensation (if any)? 
    \item[] Answer: \answerNA{} 
    \item[] Justification: No crowdsourcing or human subjects research was conducted.

\item {\bf Institutional review board (IRB) approvals or equivalent for research with human subjects}
    \item[] Question: Does the paper describe potential risks incurred by study participants, whether such risks were disclosed to the subjects, and whether Institutional Review Board (IRB) approvals (or an equivalent approval/review based on the requirements of your country or institution) were obtained?
    \item[] Answer: \answerNA{} 
    \item[] Justification: No human subjects research was conducted.

\item {\bf Declaration of LLM usage}
    \item[] Question: Does the paper describe the usage of LLMs if it is an important, original, or non-standard component of the core methods in this research? Note that if the LLM is used only for writing, editing, or formatting purposes and does \emph{not} impact the core methodology, scientific rigor, or originality of the research, declaration is not required.
    \item[] Answer: \answerNA{} 
    \item[] Justification: LLMs were not used as a core methodological component of this research.

\end{enumerate}


\begin{thebibliography}{23}
\providecommand{\natexlab}[1]{#1}
\providecommand{\url}[1]{\texttt{#1}}
\expandafter\ifx\csname urlstyle\endcsname\relax
  \providecommand{\doi}[1]{doi: #1}\else
  \providecommand{\doi}{doi: \begingroup \urlstyle{rm}\Url}\fi

\bibitem[Angelopoulos \& Bates(2022)Angelopoulos and Bates]{angelopoulos2022gentleintroductionconformalprediction}
Angelopoulos, A.~N. and Bates, S.
\newblock A gentle introduction to conformal prediction and distribution-free uncertainty quantification, 2022.
\newblock  \url{https://arxiv.org/abs/2107.07511}.

\bibitem[Angelopoulos  \& Bates(2023)Angelopoulos and Bates]{angelopoulos2023conformal}
Angelopoulos, A.~N. and Bates, S.
\newblock Conformal prediction: A gentle introduction.
\newblock \emph{Foundations and Trends{\textregistered} in Machine Learning}, 16\penalty0 (4):\penalty0 494--591, 2023.

\bibitem[Angelopoulos et~al.(2024)Angelopoulos, Bates, Fisch, Lei, and Schuster]{angelopoulos2024conformal}
Angelopoulos, A.~N., Bates, S., Fisch, A., Lei, L., and Schuster, T.
\newblock Conformal risk control.
\newblock \emph{International Conference on Learning Representations (ICLR)}, 2024.
\newblock  \url{https://openreview.net/forum?id=33XGfHLtZg}.

\bibitem[Angelopoulos et~al.(2025)Angelopoulos, Barber, and Bates]{angelopoulos2025theoreticalfoundationsconformalprediction}
Angelopoulos, A.~N., Barber, R.~F., and Bates, S.
\newblock Theoretical foundations of conformal prediction, 2025.
\newblock  \url{https://arxiv.org/abs/2411.11824}.

\bibitem[Bates et~al.(2021)Bates, Angelopoulos, Lei, Malik, and Jordan]{RCPS10.1145/3478535}
Bates, S., Angelopoulos, A., Lei, L., Malik, J., and Jordan, M.
\newblock Distribution-free, risk-controlling prediction sets.
\newblock \emph{Journal of the ACM}, 68\penalty0 (6), 2021.
\newblock  \url{https://doi.org/10.1145/3478535}.

\bibitem[Darling \& Robbins(1967)Darling and Robbins]{darlingd9fc6385-8ffb-31ca-a197-d0221f3b67d5}
Darling, D.~A. and Robbins, H.
\newblock Iterated logarithm inequalities.
\newblock \emph{Proceedings of the National Academy of Sciences of the United States of America}, 57\penalty0 (5):\penalty0 1188--1192, 1967.
\newblock  \url{http://www.jstor.org/stable/57876}.

\bibitem[Duchi(2025)]{duchi2025sampleconditional}
Duchi, J.
\newblock Sample-conditional coverage in split-conformal prediction.
\newblock In \emph{Advances in Neural Information Processing Systems}, 2025.
\newblock  \url{https://openreview.net/forum?id=aTBM5j3oyA}.

\bibitem[Duchi \& Haque(2024)Duchi and Haque]{pmlr-v247-duchi24a}
Duchi, J. and Haque, S.
\newblock An information-theoretic lower bound in time-uniform estimation.
\newblock \emph{Conference on Learning Theory}, 2024.
\newblock  \url{https://proceedings.mlr.press/v247/duchi24a.html}.

\bibitem[Durrett(2019)]{Durrett2019}
Durrett, R.
\newblock \emph{Probability: Theory and Examples}.
\newblock  Cambridge University Press, 5th edition, 2019.

\bibitem[Gauthier et al.(2025)]{gauthier2025evalues}
Gauthier, E., Bach, F., and Jordan, M. I.
\newblock E-Values Expand the Scope of Conformal Prediction.
\newblock \emph{arXiv preprint arXiv:2503.13050}, 2025.
\newblock \url{https://arxiv.org/abs/2503.13050}.

\bibitem[Ge et~al.(2025)Ge, Bastani, and Bastani]{ge2025stochastic}
Ge, H., Bastani, H., and Bastani, O.
\newblock Stochastic online conformal prediction with semi-bandit feedback.
\newblock \emph{Forty-second International Conference on Machine Learning (ICML)}, 2025.
\newblock \url{https://openreview.net/forum?id=IdRrKDsTZ8}.

\bibitem[Gibbs and Candès(2021)]{gibbs2021adaptive}
Gibbs, I. and Candès, E.
\newblock Adaptive Conformal Inference Under Distribution Shift.
\newblock \emph{Advances in Neural Information Processing Systems}, 2021.
\newblock \url{https://openreview.net/forum?id=6vaActvpcp3}.

\bibitem[Howard \& Ramdas(2022)Howard and Ramdas]{howard2022sequentialestimation}
Howard, S.~R. and Ramdas, A.
\newblock {Sequential estimation of quantiles with applications to A/B testing and best-arm identification}.
\newblock \emph{Bernoulli}, 28\penalty0 (3):\penalty0 1704--1728, 2022.
\newblock  \url{https://doi.org/10.3150/21-BEJ1388}.

\bibitem[Howard et~al.(2021)Howard, Ramdas, McAuliffe, and Sekhon]{Howard2018TimeuniformNN}
Howard, S.~R., Ramdas, A., McAuliffe, J.~D., and Sekhon, J.~S.
\newblock Time-uniform, nonparametric, nonasymptotic confidence sequences.
\newblock \emph{The Annals of Statistics}, 49\penalty0 (2):\penalty0 1055--1080, 2021.

\bibitem[Howard et~al.(2020)Howard, Ramdas, McAuliffe, and Sekhon]{Howard_2020}
Howard, S.~R., Ramdas, A., McAuliffe, J., and Sekhon, J.
\newblock Time-uniform {C}hernoff bounds via nonnegative supermartingales.
\newblock \emph{Probability Surveys}, 17:\penalty0 257--317, 2020.
\newblock \url{https://doi.org/10.1214/18-PS321}.

\bibitem[Hulsman(2022)]{hulsman2022distribution}
Hulsman, R.
\newblock Distribution-free finite-sample guarantees and split conformal prediction.
\newblock \emph{arXiv:2210.14735}, 2022.

\bibitem[{Marques F.}(2025)]{MARQUESF2025110350}
{Marques F.}, P.~C.
\newblock Universal distribution of the empirical coverage in split conformal prediction.
\newblock \emph{Statistics \& Probability Letters}, 219:\penalty0 110350, 2025.
\newblock  \url{https://doi.org/10.1016/j.spl.2024.110350}.

\bibitem[Pournaderi \& Xiang(2026)Pournaderi and Xiang]{pournaderi2026trainingconditional}
Pournaderi, M. and Xiang, Y.
\newblock Training-conditional coverage bounds under covariate shift.
\newblock \emph{Transactions on Machine Learning Research}, 2026.
\newblock  \url{https://openreview.net/forum?id=F6hHT3qWxT}.

\bibitem[Ramdas et~al.(2023)Ramdas, Grünwald, Vovk, and Shafer]{overview}
Ramdas, A., Grünwald, P., Vovk, V., and Shafer, G.
\newblock Game-theoretic statistics and safe anytime-valid inference.
\newblock \emph{Statistical Science}, 38\penalty0 (4), 576--601, 2023.
\newblock \url{https://doi.org/10.1214/23-STS894}.

\bibitem[Shimodaira(2000)]{SHIMODAIRA2000227}
Shimodaira, H.
\newblock Improving predictive inference under covariate shift by weighting the log-likelihood function.
\newblock \emph{Journal of Statistical Planning and Inference}, 90\penalty0 (2):\penalty0 227--244, 2000.
\newblock  \url{https://doi.org/10.1016/S0378-3758(00)00115-4}.

\bibitem[Vovk(2012)]{pmlr-v25-vovk12}
Vovk, V.
\newblock Conditional validity of inductive conformal predictors.
\newblock \emph{Asian Conference on Machine Learning}, 2012.
\newblock  \url{https://proceedings.mlr.press/v25/vovk12.html}.

\bibitem[Vovk et~al.(2005)Vovk, Gammerman, and Shafer]{vovk2005algorithmic}
Vovk, V., Gammerman, A., and Shafer, G.
\newblock \emph{Algorithmic learning in a random world}.
\newblock Springer, 2005.

\bibitem[Waudby-Smith \& Ramdas(2024)Waudby-Smith and Ramdas]{10.1093/jrsssb/qkad009}
Waudby-Smith, I. and Ramdas, A.
\newblock Estimating means of bounded random variables by betting.
\newblock \emph{Journal of the Royal Statistical Society Series B: Statistical Methodology}, 86\penalty0 (1):\penalty0 1--27, 2024.
\newblock  \url{https://doi.org/10.1093/jrsssb/qkad009}.

\bibitem[Waudby-Smith et~al.(2024)Waudby-Smith, Wu, Ramdas, Karampatziakis, and Mineiro]{offpolicy10.1145/3643693}
Waudby-Smith, I., Wu, L., Ramdas, A., Karampatziakis, N., and Mineiro, P.
\newblock Anytime-valid off-policy inference for contextual bandits.
\newblock \emph{ACM / IMS J. Data Sci.}, 1\penalty0 (3), 2024.
\newblock  \url{https://doi.org/10.1145/3643693}.

\bibitem[Xu et~al.(2024)Xu, Karampatziakis, and Mineiro]{xu2024active}
Xu, Z., Karampatziakis, N., and Mineiro, P.
\newblock Active, anytime-valid risk controlling prediction sets.
\newblock \emph{Advances in Neural Information Processing Systems}, 2024.
\newblock  \url{https://openreview.net/forum?id=4ZH48aGD60}.

\bibitem[Zwart(2025)]{zwart2025probabilisticconformalcoverageguarantees}
Zwart, P.~H.
\newblock Probabilistic conformal coverage guarantees in small-data settings, 2025.
\newblock  \url{https://arxiv.org/abs/2509.15349}.

\end{thebibliography}
\end{document}